\pdfoutput=1
\documentclass[11pt]{article}

\usepackage[dvipsnames]{xcolor}
\definecolor{darkblue}{rgb}{0, 0, 0.5}

\usepackage[acceptedWithA]{tacl2021v1}  

\usepackage[T1]{fontenc}
\usepackage{times}
\usepackage[scaled=1.06]{inconsolata}
\usepackage[scaled=0.9]{helvet}

\usepackage{amsmath}
\usepackage{amssymb}
\usepackage{amsthm}
\usepackage{mathtools}
\usepackage{mleftright}

\usepackage{multirow}

\usepackage{newtxmath} %

\usepackage{natbib}
\usepackage{hyperref}
\usepackage{doi} %

\usepackage{booktabs}
\usepackage{url}
\usepackage{caption}
\usepackage{subcaption}

\usepackage[capitalize,nameinlink]{cleveref} %

\newtheorem{theorem}{Theorem}
\newtheorem{lemma}[theorem]{Lemma}
\newtheorem{corollary}[theorem]{Corollary}

\newtheorem{proposition}[theorem]{Proposition}
\theoremstyle{definition}
\newtheorem{definition}[theorem]{Definition}

\usepackage[inline]{enumitem}
\setlist[itemize]{topsep=0pt,partopsep=0pt,itemsep=0pt,parsep=0pt,leftmargin=1em}
\setlist[itemize,1]{label=--}
\setlist[itemize,2]{label=\textbullet}
\setlist[itemize,3]{label=*}
\setlist[description]{topsep=0pt,partopsep=0pt,itemsep=0pt,parsep=0pt,leftmargin=1em}
\setlist[enumerate]{topsep=0pt,partopsep=0pt,itemsep=0pt,parsep=0pt,leftmargin=1.75em}
\setlist[enumerate,1]{label=\arabic*.,ref=\arabic*}
\setlist[enumerate,2]{label=(\alph*)}
\setlist[enumerate,3]{label=\roman*.}

\newlist{subtheorems}{enumerate}{1}
\setlist[subtheorems]{topsep=0pt,partopsep=0pt,itemsep=0pt,parsep=0pt,leftmargin=1.75em}
\setlist[subtheorems,1]{label=(\alph*),ref=\alph*}

\AtBeginDocument{\suppressfloats[t]}

\usepackage{graphicx} %

\usepackage{tikz}
\usetikzlibrary{arrows,arrows.meta,positioning,shapes,decorations,automata,backgrounds,petri}
\usepackage{relsize}
\usetikzlibrary{fit}
\tikzset{fontscale/.style = {font=\relsize{#1}}}

\DeclareMathOperator{\rightmost}{{\mathbf{\blacktriangleright}}}
\newcommand{\att}[6]{\ensuremath{{#1}_{#2} \left[#3, #4\right] \; #5 : #6}}

\DeclareMathOperator{\SRASP}{\probclass{S-RASP}}

\newcommand{\ind}[1]{\mathbb{I}\left[#1\right]}

\newcommand{\lpto}{\overset{\textnormal{lp}}{\to}}

\newcommand{\R}{\mathbb{R}}

\newcommand{\probclass}[1]{\mathsf{#1}}

\newcommand{\Mon}{\probclass{Mon}}

\newcommand{\UHAT}{\probclass{UHAT}}

\newcommand{\AHAT}{\probclass{AHAT}}
\newcommand{\SMAT}{\probclass{SMAT}}
\newcommand{\tauSMAT}{\tau\textsf{-}{\mathsf{SMAT}}}

\newcommand{\LTL}{\probclass{LTL}}

\newcommand{\LTLCMon}{\probclass{\mathsf{LTL[\countl,\countr,+,\Mon]}}}
\newcommand{\poly}{\probclass{poly}}

\newcommand{\prob}[1]{\ensuremath{\textsf{\uppercase{#1}}}}

\newcommand{\softmax}{\mathop{\textnormal{softmax}}\nolimits}
\newcommand{\lhardmax}{\mathop{\textnormal{lhardmax}}}
\newcommand{\rhardmax}{\mathop{\textnormal{rhardmax}}}
\newcommand{\hardmax}{\mathop{\textnormal{hardmax}}}

\newcommand{\avghardmax}{\mathop{\textnormal{ahardmax}}}

\newcommand{\ReLU}{\textnormal{ReLU}}
\newcommand\norm[1]{\left\lVert#1\right\rVert_1}

\newcommand{\qproj}{\mat{W}^{(\textnormal{Q})}}
\newcommand{\qprojl}[1]{\mat{W}^{(\textnormal{Q},#1)}}
\newcommand{\kproj}{\mat{W}^{(\textnormal{K})}}
\newcommand{\kprojl}[1]{\mat{W}^{(\textnormal{K},#1)}}
\newcommand{\vproj}{\mat{W}^{(\textnormal{V})}}
\newcommand{\vprojl}[1]{\mat{W}^{(\textnormal{V},#1)}}
\newcommand{\ffwone}{\mat{W}_1}
\newcommand{\ffwonel}[1]{\mat{W}_1^{(#1)}}
\newcommand{\ffbonel}[1]{\vec{b}_1^{(#1)}}
\newcommand{\ffwtwo}{\mat{W}_2}
\newcommand{\ffwtwol}[1]{\mat{W}_2^{(#1)}}
\newcommand{\ffbtwol}[1]{\vec{b}_2^{(#1)}}
\newcommand{\qvec}[1]{\vec{q}_{#1}}
\newcommand{\kvec}[1]{\vec{k}_{#1}}
\newcommand{\vvec}[1]{\vec{v}_{#1}}

\makeatletter
\newcommand{\oset}[2]{%
  {\mathop{#2}\limits^{\vbox to -.5\ex@{\kern-\tw@\ex@
   \hbox{\scriptsize $#1$}\vss}}}}
\makeatother
\newcommand{\countl}{\ensuremath{\oset{\leftharpoonup}{\#}}}
\newcommand{\countr}{\ensuremath{\oset{\rightharpoonup}{\#}}}

\newcommand{\since}{\mathop{\textnormal{\textbf{since}}}}
\newcommand{\until}{\mathop{\textnormal{\textbf{until}}}}

\newcommand{\mat}[1]{\mathbf{#1}}
\renewcommand{\vec}[1]{\mathbf{#1}}
\newcommand{\str}[1]{\mathbf{#1}}

\newcommand{\TLPN}{\ensuremath{\mathsf{TL}[\ominus,\oplus]}}
\newcommand{\TLP}{\ensuremath{\mathsf{TL}[\ominus]}}
\newcommand{\TLS}{\ensuremath{\mathsf{TL}[\since]}}
\newcommand{\TLSU}{\ensuremath{\mathsf{TL}[\since,\until]}}
\newcommand{\TLPNSU}{\ensuremath{\mathsf{TL}[\ominus,\oplus,\since,\until]}}

\newcommand{\TLCCMon}{\ensuremath{\mathsf{TL}[\countl,\countr,+,\textnormal{\textnormal{\textsf{Mon}}}]}}
\newcommand{\TLCMon}{\ensuremath{\mathsf{TL}[\countl,+,\textnormal{\textnormal{\textsf{Mon}}}]}}

\newcommand{\TLPS}{\ensuremath{\mathsf{TL}[\ominus,\since]}}

\newcommand{\TLPSCMon}{\ensuremath{\mathsf{TL}[\ominus,\since,\countl,+,\textnormal{\textsf{Mon}}]}}

\newcommand{\prv}{\mathop{\ominus}}
\newcommand{\nxt}{\mathop{\oplus}}

\newcommand{\tlscale}{\gamma} %
\newcommand{\tablerow}{t}

\newenvironment{raspcode}{\par\vspace*{-2ex}\begin{small}}{\end{small}\ignorespacesafterend}

\newcommand{\attrdefault}[5]{\att{\rightmost}{#1}{#2}{#3}{#4}{#5}}

\newcommand{\attshort}[4]{\ensuremath{{#1}_{#2} \left[#3\right] \; #4}}
\DeclareMathOperator{\sumsymb}{{\bf psum}}
\newcommand{\attsum}[3]{\attshort{\sumsymb}{#1}{#2}{#3}}

\newcommand{\rechar}[1]{\text{`{#1}'}}

\newcommand{\true}{\top}
\newcommand{\false}{\bot}
\newcommand{\vecname}[1]{\ensuremath{\mathtt{#1}}}

\newcommand{\inputx}{\vecname{in}}
\newcommand{\outputy}{\vecname{out}}

\newcommand{\firstline}[1]{\hspace{1em}&\hspace{-1em}#1}

\newcommand{\Kleft}{\oset{\leftharpoonup}{K}}
\newcommand{\Kright}{\oset{\rightharpoonup}{K}}

\title{Simulating Hard Attention Using Soft Attention}

\author{
Andy Yang \\
University of Notre Dame \\
\href{mailto:ayang4@nd.edu}{\tt ayang4@nd.edu}
\And
Lena Strobl \\
Umeå University \\
\href{mailto:lena.strobl@umu.se}{\tt lena.strobl@umu.se} \\ 
\AND 
David Chiang \\
University of Notre Dame \\
\href{mailto:dchiang@nd.edu}{\tt dchiang@nd.edu}
\And
Dana Angluin \\
Yale University \\
\href{mailto:dana.angluin@yale.edu}{\tt dana.angluin@yale.edu}
}

\begin{document}

\maketitle

\begin{abstract}
We study conditions under which transformers using soft attention can simulate hard attention, that is, effectively focus all attention on a subset of positions. 
First, we examine several subclasses of languages recognized by hard-attention transformers, which can be defined in variants of linear temporal logic. We demonstrate how soft-attention transformers can compute formulas of these logics using unbounded positional embeddings or temperature scaling. 
Second, we demonstrate how temperature scaling allows softmax transformers to simulate general hard-attention transformers, using a temperature that depends on the minimum gap between the maximum attention scores and other attention scores.
\end{abstract}

\section{Introduction}
\label{sec:introduction}

A central element of transformers \citep{vaswani2017attention} is \emph{attention}, in which each position computes a weighted average of values from all unmasked positions. In standard attention, which we call \emph{soft} attention, the attention weights are computed by the $\softmax$ function and cannot be exactly $0$ (unless there is masking) or exactly $1$ (unless there is only one position). This is appropriate in applications like machine translation \citep{bahdanau2015neural}, where attention models the often fuzzy relationship of target words to source words.
But there is a tension between soft attention and discrete tasks like integer arithmetic, algebra, logical reasoning, finite automata, Turing machines, and computer programs.
Researchers trying to show that transformers can solve such tasks have often turned to other definitions of attention.

\Citet{hahn-2020-theoretical} defined what is now called \emph{unique-hard attention}, in which all attention is on a single position attaining the maximum attention score, and zero attention is paid to all other positions. 
\Citet{angluin2023masked} proved that unique-hard attention transformers ($\UHAT$s) with no position information besides masking recognize exactly the \emph{star-free regular languages}, and equivalences for other variants as well.
Although this means that $\UHAT$s can solve some interesting problems like the multi-query associative recall task \citep{friedman2023learning}, the star-free regular languages are a rather restricted class, so one might expect that $\UHAT$s are less expressive than standard soft-attention transformers ($\SMAT$s). But, perhaps surprisingly, it has not been shown previously that $\UHAT$s can be simulated by $\SMAT$s.

\Citet{pérez2019turingcompletenessmodernneural}, to prove that transformers with chain-of-thought are Turing-complete, introduced what is now called \emph{average-hard attention} (or \emph{saturated attention}): given a score vector, equal attention is paid to each position in which the maximum attention score is attained, and zero attention is paid to all other positions.
Average-hard attention transformers ($\AHAT$s) have been constructed to solve a variety of discrete tasks: matching parentheses \citep{yao2023selfattentionnetworksprocessbounded}, simulating an $n$-gram model \citep{svete-cotterell-2024-transformers}, simulating linear temporal logic with counting terms \citep{barcelo2023logical}, simulating the programming language S-RASP
\citep{strobl2024transformers}, and others. 
For many of these tasks, no construction using soft attention is known.

In this paper, we contribute new results that help to answer the question: \emph{Under what conditions can soft attention simulate hard attention?}  
We are interested in simulations of hard attention that are \emph{general} -- that is, for a class of hard attention transformers, not just for a particular language -- and \emph{parameter-uniform},
that is, the number and values of the parameters do not depend on the input length $n$.
This ensures that the same model and parameters can be used for arbitrary $n$.

However, a result of \citet[Lemma 5]{hahn-2020-theoretical} stands in the way.
If a transformer has
\begin{enumerate*}[label=(\arabic*),afterlabel={},itemjoin={\ }]
\item\label{item:softmax} %
soft attention,
\item\label{item:uniform} parameters whose number and values do not depend on $n$, 
\item\label{item:lipschitz} Lipschitz-continuous position-wise operations, and 
\item\label{item:bounded} bounded position embeddings (PEs), 
\end{enumerate*} then a change in one input symbol results in only an $O(1/n)$ change in any output activation. This is in contrast to $\AHAT$s, where a change in one input symbol can result in a $\Theta(1)$ change in an output activation.

Why is this property important? 
For example, \citet{chiang-cholak-2022-overcoming} construct a $\SMAT$ for \prob{Parity} (deciding whether the number of $1$'s in a string is odd), but as $n$ increases, the gap between the output for acceptance and rejection approaches~$0$, as it must. %
However, this is a problem for composability: If one module has discrete output values that are separated by $O(1/n)$, and another module expects discrete inputs that are separated by $\Theta(1)$, then the two modules cannot be composed.

The solutions that we are aware of each circumvent Hahn's lemma by dropping one of the above assumptions:
\begin{enumerate}
\item[\ref{item:softmax}] Temperature scaling, in which attention scores are scaled by a function of $n$ before the softmax. \Citet{chiang-cholak-2022-overcoming} and \citet{nakanishi2025scalablesoftmaxsuperiorattention} both experiment with $O(1/\log n)$ temperature and find that it can improve length generalization.
\item[\ref{item:uniform}] Rounding, used by \citet{sanford2024transformersparallelcomputationlogarithmic} to correct approximation errors. They round to $\Theta(\log n)$ bits, which requires a feedforward network with depth or width depending on $n$ (\cref{lemma-relu-mlp-rounding}), for a construction that is general, but not parameter-uniform. Similarly, \citet{li-etal-2024-empowers} round to constant precision, but the embedding size depends on $n$.
\item[\ref{item:lipschitz}] Layer normalization, as originally defined \citep{ba2016layernormalization}, is not Lipschitz-continuous, and can be used to increase sensitivity \citep{chiang-cholak-2022-overcoming,yao2023selfattentionnetworksprocessbounded,merrill2024illusionstatestatespacemodels,yang2024counting}.
\item[\ref{item:bounded}] Unbounded position embeddings, whose norm may grow as a function of $n$.
\end{enumerate}

Here, we consider general and parameter-uniform simulations of hard attention by soft attention, using unbounded position embeddings or temperature scaling (methods \labelcref{item:softmax,item:bounded} above).
Our results are summarized in \cref{fig:results}.

\begin{figure}
    \centering
    \tikzset{edgelabel/.style={align=center,font={\footnotesize}}}
    \begin{tikzpicture}[>={Stealth[length=2mm,width=1.5mm]},thick,x=1.5cm,y=1.5cm]
        \node[anchor=base] (tlpn) at (-0.1,0) {\TLPN};
        \node[anchor=base] (tlsu) at (1.3,0) {\TLSU};
        \node[anchor=base] (tlccmon) at (3.4,0) {\TLCCMon};
        \node[anchor=base] (ltl) at (0.5,0.75) {$\LTL$};
        \node[anchor=base] (ltlcmon) at (2, 1.3) {$\LTLCMon$};
        \node (ahat) at (2,2.4) {$\AHAT$};
        \node (smat) at (2,3.3) {$\tau\textsf{-}\SMAT$};

        \draw[->,dashed] (tlpn) to (ltl);
        \draw[->,dashed] (tlsu) to (ltl);
        \draw[->,dashed] (ltl) to (ltlcmon);
        \draw[->,dashed] (tlccmon) to (ltlcmon);
        \draw[->,dashed] (ltlcmon) to node[auto=right,text width=1.75cm,edgelabel,align=left] {\raggedright\citep{barcelo2023logical}} (ahat);

        \begin{scope}[->,thick,rounded corners]
        \draw (tlccmon) edge[bend right] node[edgelabel,right] {\S\ref{sec:utlcmon}} (smat);
        \draw (ahat) to node[edgelabel,auto=right] {\S\ref{sec:uuhat_to_smat_temp}} (smat);
        \draw (tlpn) edge[bend left=50] node[edgelabel,auto=left] {\S\ref{sec:prev-and-next}} (smat);
        \draw (tlsu) edge[bend left=50] node[edgelabel,auto=left]{\S\ref{sec:since_until}} (smat);
        \end{scope}
    \end{tikzpicture} \\[2ex]
    \begin{center} \small
        \begin{tabular}{llll}
            \toprule
            &\textbf{Simulatee} & \textbf{Temp.} $\tau(n)$ & \textbf{PE} \\
            \midrule
            \multirow{2}{*}{\S\ref{sec:prev-and-next}} & \multirow{2}{*}{$\TLPN$} & $O(1/n)$ & $O(1)$\\
            && $O(1)$ & $O(n)$ \\
            \midrule
            \multirow{2}{*}{\S\ref{sec:since_until}} & \multirow{2}{*}{$\TLSU$} & $O(1/n)$ & $O(1)$\\
            && $O(1)$ & $O(n)$ \\
            \midrule
            \multirow{2}{*}{\S\ref{sec:utlcmon}} & \multirow{2}{*}{$\TLCCMon$} & $O(1/n)$ & $O(n^2)$\\
            && $O(1)$ & $O(n^3)$ \\
            \midrule
            \S\ref{sec:uuhat_to_smat_temp} &$\AHAT$ & \multicolumn{2}{c}{See \cref{thm:smat_approx_hard_attention}} \\
            \bottomrule
        \end{tabular}
    \end{center}
    \caption{Simulation by $\tauSMAT$. Solid arrows denote results proved in this paper. Key: $\tau$ = temperature;
    PE = position embedding.
    }
    \label{fig:results}
\end{figure}

As mentioned above, \citet{barcelo2023logical} proved that linear temporal logic with counting, or $\LTLCMon$, can be simulated by an $\AHAT$. Thus, $\LTLCMon$ defines a subclass of the languages recognized by $\AHAT$s, and various fragments of $\LTLCMon$ define smaller subclasses.
Our first set of results (\cref{sec:soft-attention-LTLCmon}) examines several such subclasses -- among them, the languages recognized by $\UHAT$s -- and shows that they can be simulated by $\SMAT$s using either inverse polynomial temperature or polynomially bounded position embeddings.

Our second result (\cref{sec:uuhat_to_smat_temp}) concerns general $\AHAT$s and their direct simulation by $\SMAT$s. We classify $\AHAT$s according to the \emph{gap} $\gamma(n)$ that separates, for inputs of length $n$, the maximum attention score or scores from all other scores \citep{edelman2022inductive}.
In the literature on transformer expressivity, constructions of $\AHAT$s generally have gap~$\gamma(n) \in \Omega(1/n^k)$ for small constants $k$ (\cref{sec:applications}). 
We prove that any $\AHAT$ with gap $\gamma(n)$ can be approximately simulated by a $\SMAT$ with the same parameters, using soft attention scaled by a temperature not much lower than $\gamma(n)$.

The implications of our results for the expressivity and learnability of transformers in practice are discussed in \cref{sec:conclusion}.

\section{Preliminaries}
\label{sec:preliminaries}

\subsection{Notation}

We write $[n]$ for the set $\{1, \ldots, n\}$. For any Boolean value $b$, we let $\ind{b} = 1$ if $b$ is true and $0$ if $b$ is false.

If $X$ and $Y$ are sets, we write $X^+$ for the set of non-empty sequences over $X$, and we write $f \colon X^+ \lpto Y^+$ to state that $f$ is \emph{length-preserving}, that is, it maps a sequence $(x_1, \ldots, x_n) \in X^+$ to a sequence of the same length $(y_1, \ldots, y_n) \in Y^+$.

If $\mat{A}$ is a matrix, we write $\vec{A}_{i,*}$ for the $i$-th row of $\mat{A}$ and $\vec{A}_{*,j}$ for the $j$-th column of $\mat{A}$.

\subsection{Transformers}
\label{subsec:transformer}

A \emph{transformer} is a neural network that, in this paper, defines a length-preserving mapping from strings to strings. An input layer maps a string to a sequence of vectors. 
Then the network processes sequences of vectors through multiple layers, each consisting of a self-attention sublayer followed by a feed-forward sublayer.  
Finally, an output layer maps the sequence of vectors to a string.

\paragraph{Input}

Let $\Sigma$ be an input alphabet, and let $\str{w} = w_1 w_2 \cdots w_n \in \Sigma^+$ be an input sequence of length~$n$.
Each token $w_i$ is mapped to a vector $\vec{x}_i = \text{WE}(w_i) + \text{PE}_n(i)$, where $\text{WE} \colon \Sigma \to \R^{d}$ is a \emph{word embedding} and $\text{PE}_n \colon \mathbb{N} \to \R^d$ a \emph{position embedding}.
Thus the input sequence is represented as a sequence of vectors $(\vec{x}_{1}, \cdots, \vec{x}_{n}) \in (\R^d)^+$.

\paragraph{Transformer Layers}

We initialize $\vec{h}^{(0)}_i = \vec{x}_i$ for all $i$. For each layer $\ell = 1, \dots, L$ and position $i = 1, \ldots n$, we compute
    \[
    \vec{c}^{(\ell)}_i = \mathrm{SA}^{(\ell)} \left( \vec{h}^{(\ell-1)}_1, \ldots, \vec{h}^{(\ell-1)}_n \right)_i + \vec{h}^{(\ell-1)}_i
    \]
    where $\mathrm{SA}^{(\ell)}$ is the \emph{self-attention} function defined below in \cref{def:attn}, with parameters $\qprojl{\ell}$, $\kprojl{\ell}$, $\vprojl{\ell}$, and weighting function~$\mathcal{S}^{(\ell)}$.
    Then we apply a \emph{feed-forward network} (FFN):
    \[
    \vec{h}^{(\ell)}_i = \ffwtwol{\ell} \, \sigma \left( \ffwonel{\ell} \, \vec{c}^{(\ell)}_i + \ffbonel{\ell} \right ) + \ffbtwol{\ell} + \vec{c}^{(\ell)}_i
    \]
    where $\sigma$ is the $\ReLU$ activation function, $\ffwonel{\ell} \in \R^{d_{\textrm{f}} \times d}$, $\ffwtwol{\ell} \in \R^{d \times d_{\textrm{f}}}$, $\ffbonel{\ell} \in \R^{d_{\textrm{f}}}$, and $\ffbtwol{\ell} \in \R^{d}$.

    The ${}+\vec{h}^{(\ell-1)}_i$ and ${}+\vec{c}^{(\ell)}_i$ terms above are called \emph{residual connections}. In our constructions, they are useful for making values computed at earlier steps to be available at all later steps.

    Transformer layers are usually implemented with \emph{layer normalization}, but we omit this here, as this is a separate method (method \ref{item:lipschitz} in \cref{sec:introduction}) for overcoming limitations of soft attention, which may interact in complicated ways with temperature scaling and unbounded PEs.

\paragraph{Output}

After $L$ layers, the transformer produces an output sequence $\left(\vec{y}_1, \ldots, \vec{y}_n\right) \in (\R^{d})^+$, where $\vec{y}_i = \vec{h}^{(L)}_{i}$ for all $i$.
To map these vectors to output symbols, we assume, as is common in theoretical research, that each $\vec{y}_i$ is exactly the embedding of a symbol in the output alphabet.
To accept or reject strings, we look at the output symbol at position~$n$ and designate two symbols as accept and reject decisions. Note that this implies a fixed-size gap between the outputs for acceptance and rejection.

\paragraph{Self-Attention}
A self-attention layer computes weighted sums of \emph{value} vectors at all positions, where the weights are determined by \emph{query} and \emph{key} vectors.

\begin{definition}[Self-Attention]\label{def:attn}
A self-attention layer with linear transformations
\begin{align*}
\qproj, \kproj \colon \R^d &\to \R^{d_{\textrm{k}}} \\
\vproj \colon \R^d &\to \R^{d}
\end{align*}
and a weighting function
\[
\mathcal{S} \colon \R^+ \lpto \R^+
\]
is a function
\begin{align*}
    \mathrm{SA} \colon (\R^d)^+ &\lpto (\R^d)^+ \\
    \mathrm{SA}\mleft(\vec{h}_{1}, \ldots, \vec{h}_{n} \mright) &= (\vec{c}_1, \ldots, \vec{c}_n)
\end{align*}
where, for positions $i$ and $j$:
\begin{align*}
\qvec{i} &= \qproj \vec{h}_{i} &
\kvec{j} &= \kproj \vec{h}_{j} &
\vvec{j} &= \vproj \vec{h}_{j} \\
s_{ij} &= \frac{\qvec{i}^\top \kvec{j}}{\sqrt{d_{\textrm{k}}}} &
\alpha_{i,*} &= \mathcal{S}(s_{i,*}) &
\vec{c}_i &= \sum_{j=1}^{n} \alpha_{ij} \vvec{j}.
\end{align*}

We consider only single-head attention layers for simplicity, as a $k$-head attention layer can be simulated using $k$ layers of single-head attention (each layer computes the value of one head, and residual connections are used to accumulate the sum of the heads).
\end{definition}

\paragraph{Masking} In \emph{future-masked attention}, the scores are redefined to
\[ s_{ij} = \begin{cases} \dfrac{\qvec{i}^\top \kvec{j}}{\sqrt{d_{\textrm{k}}}} & j \le i \\
-\infty & \text{otherwise}
\end{cases}
\]
and we define $\exp(-\infty) = 0$.
Similarly, in \emph{past-masked attention}, we have $s_{ij} = -\infty$ iff $j < i$.

Some of our proofs construct transformers that use both future-masked and past-masked attention \citep[also cf.][]{yao2023selfattentionnetworksprocessbounded}. In practice, mixed attention masks have been employed by applying different masks to different attention heads \citep{shen+:2018} or by applying multiple masks with learnable weights \citep{mcdonald-chiang-2021-syntax,sible-chiang-2024}.

\paragraph{Weighting Function}
The weighting function $\mathcal{S} \colon \R^+ \lpto \R^+$ computes the attention weights $\alpha_{i,*}$ based on the attention scores $s_{i,*}$.
A common choice is the \emph{softmax} function:
\[
\left[\softmax(s_1, \ldots, s_n)\right]_j = \frac{e^{s_j}}{\sum_{k=1}^{n} e^{s_k}}.
\]
In hard attention, the attention weights are assigned to focus all attention on the maximum-scoring position or positions.

\begin{definition}[Hardmax]
\label{def:hardmax}
The leftmost, rightmost, and average-hardmax functions are defined as:
\begin{align*}
I(\vec{s}) &= \{ i \in [|\vec{s}|] \mid s_i = \max \vec{s} \} \\
[\lhardmax\vec{s}]_i &= \mathbb{I}[ i = \min I(\vec{s}) ] \\
[\rhardmax\vec{s}]_i &= \mathbb{I}[ i = \max I(\vec{s}) ] \\
[\avghardmax\vec{s}]_i &= \frac{1}{|I(\vec{s})|} \mathbb{I}[ i \in I(\vec{s}) ].
\end{align*}
\end{definition}
The $\lhardmax$ and $\rhardmax$ functions return a one-hot vector with a $1$ at the position of the leftmost or rightmost maximal element, respectively. 
The $\avghardmax$ function pays equal attention to all the maximal elements.

\paragraph{Temperature Scaling}

We also consider scaling the attention scores by an inverse \emph{temperature} before applying the softmax.

\begin{definition}[Temperature-Dependent Softmax]
\label{def:temperature-dependent-softmax}
For a temperature $\tau > 0$, the temperature-dependent softmax function $\softmax_{\tau} \colon \R^+ \lpto \R^+$ is defined as:
\[
[\softmax_{\tau}(s_i, \ldots, s_n)]_j = \frac{e^{s_j / \tau}}{\sum_{k=1}^{n} e^{s_k / \tau}}.
\]
When $\tau = 1$, this reduces to standard softmax. 

\end{definition}

\begin{definition}\label{def:transformer_abbreviations}
    We categorize transformers by the weighting functions used in their attention layers: 
    \begin{center}
    \begin{tabular}{r@{ }c@{ }l}
        $\SMAT$ & uses & $\softmax$ \\
        $\tauSMAT$ & uses & $\softmax_\tau$ \\
        $\AHAT$ & uses & $\avghardmax$ \\
        $\UHAT$ & uses & both $\lhardmax$ and $\rhardmax$.
    \end{tabular}
    \end{center}
\end{definition}

\section{Basic Approximations}
\label{sec:approximations}

\begin{figure} 
\centering
\tikzset{every node/.style={font={\scriptsize}}}
\tikzset{tick/.style={gray}}
\tikzset{score/.style={very thick}}
\begin{tabular}{cc}
\begin{tikzpicture}[x=5mm,y=5mm]
\draw[tick] (-2,0) -- (3,0) node[right] {$\phantom-0$};
\draw[tick] (-2,-1) -- (3,-1) node[right] {$-\gamma$};
\draw[tick] (-2,-2) -- (3,-2) node[right] {$-2\gamma$};
\draw[score] (-2,-2.2) -- (-1,-2.2) -- (-1,-1.2) -- (0,-1.2) -- (0,0) -- (1,0) -- (1,-1) -- (2,-1) -- (2,-2.1) -- (3,-2.1);
\node[align=center,font={\small}] at (0.5,-4) {ziggurat \\ (\cref{lem:softmax_bound_constant})};
\end{tikzpicture} &
\begin{tikzpicture}[x=5mm,y=2.5mm]
\draw[tick] (-2,0) -- (3,0) node[right] {$9\gamma$};
\draw[tick] (-2,-2) -- (3,-2) node[right] {$7\gamma$};
\draw[tick] (-2,-4) -- (3,-4) node[right] {$5\gamma$};
\draw[score] (-2,-4) -- (-1,-4) -- (-1,-1) -- (0,-1) -- (0,0) -- (1,0) -- (1,-1) -- (2,-1) -- (2,-4) -- (3,-4);
\node[align=center,font={\small}] at (0.5,-8) {parabola \\ (\cref{table-lookup-approx})};
\end{tikzpicture} 
\end{tabular}
\\[3ex]
\begin{tabular}{@{}c@{\!\!\!}c@{\!\!\!}c@{}}
\multicolumn{3}{c}{\begin{tikzpicture}[x=3mm,y=3mm,baseline=-6.5mm]
\draw[tick] (-2,0) -- (3,0) node[right] {$\phantom-0$};
\draw[tick] (-2,-1) -- (3,-1) node[right] {$-\gamma$};
\draw[score] (-2,-1) -- (-1,-1) -- (-1,0) -- (2,0) -- (2,-1) -- (3,-1);
\node[align=center,font={\small}] at (0.5,-3) {no tie-breaker \\ $\overbrace{\hspace{2.75in}}$};
\end{tikzpicture}} \\
\begin{tikzpicture}[x=3mm,y=3mm]
\draw[tick] (-2,0) -- (3,0) node[right] {$\phantom-0$};
\draw[tick] (-2,-1) -- (3,-1) node[right] {$-25\gamma$};
\draw[tick] (-2,-2) -- (3,-2) node[right] {$-50\gamma$};
\draw[score] (-2,-2) -- (-1,-2) -- (-1,-0.5) -- (0,-0.5) -- (0,-0.333) -- (1,-0.333) -- (1,-0.25) -- (2,-0.25) -- (2,-1.2) -- (3,-1.2);
\node[align=center,font={\small}] at (0.5,-5) {$-1/i$ \\ tie-breaker \\ (\cref{thm:tie}\ref{thm:tie_causal})};
\end{tikzpicture} &
\begin{tikzpicture}[x=3mm,y=3mm]
\draw[tick] (-2,1) -- (3,1) node[right] {$\phantom-10\gamma$};
\draw[tick] (-2,0) -- (3,0) node[right] {$\phantom-0$};
\draw[tick] (-2,-1) -- (3,-1) node[right] {$-10\gamma$};
\draw[score] (-2,-0.8) -- (-1,-0.8) -- (-1,0.4) -- (0,0.4) -- (0,0.6) -- (1,0.6) -- (1,0.8) -- (2,0.8) -- (2,0) -- (3,0);
\node[align=center,font={\small}] at (0.5,-4) {$+i/n$ \\ tie-breaker \\ (\cref{thm:tie}\ref{thm:tie_rightmost})};
\end{tikzpicture} &
\begin{tikzpicture}[x=3mm,y=3mm]
\draw[tick] (-2,0) -- (3,0) node[right] {$\phantom-0$};
\draw[tick] (-2,-1) -- (3,-1) node[right] {$-10\gamma$};
\draw[tick] (-2,-2) -- (3,-2) node[right] {$-20\gamma$};
\draw[score] (-2,-1.2) -- (-1,-1.2) -- (-1,-0.4) -- (0,-0.4) -- (0,-0.6) -- (1,-0.6) -- (1,-0.8) -- (2,-0.8) -- (2,-2) -- (3,-2);
\node[align=center,font={\small}] at (0.5,-5) {$-i/n$ \\ tie-breaker \\ (\cref{thm:tie}\ref{thm:tie_leftmost})};
\end{tikzpicture} 
\end{tabular}
\caption{Illustration of the various attention score patterns in \cref{sec:approximations}. %
}
\label{fig:approximations}
\end{figure}

In this section, we introduce some techniques that will be used throughout the paper. These techniques all serve to bound and correct the difference between hard attention and soft attention.

\begin{definition} \label{def:tieless}
For any vector of attention scores $\vec{s} \in \R^+$, let $s_{\textnormal{max}} = \max_i s_i$.
We say that $\vec{s}$ is \emph{tieless} if $|\{i \mid s_i = s_{\textnormal{max}}\}| = 1$, and $\vec{s}$ has \emph{gap $\gamma$} if for all $i$ such that $s_i \ne s_{\textnormal{max}}$, we have $s_i \le s_{\textnormal{max}}-\gamma$.
\end{definition}
If $\vec{s}$ is tieless, then $\lhardmax\vec{s} = \rhardmax\vec{s} = \avghardmax\vec{s}$, and we simply write $\hardmax\vec{s}$ for all three.

\Citet{edelman2022inductive} prove a bound on the difference between $\softmax$ and $\avghardmax$ for attention scores with gap $\gamma$. Here, we prove several stronger bounds under stronger assumptions.

\subsection{Ziggurat Attention Scores}

First, we give a bound for attention scores that are not only tieless with gap $\tlscale$, but also decrease by at least $\tlscale$ at every position in a ``ziggurat'' pattern (\cref{fig:approximations}).
\begin{lemma} 
    \label{lem:softmax_bound_constant}
    Suppose an attention layer has scores $\vec{s} = (s_1, \ldots, s_n)$, and there is a $j^* \in [n]$ such that for all $j \in [n]$, $s_{j} \le s_{j^*} - |j-j^*|\gamma$. 
    Let $v_1, \ldots, v_n$ be the attention values, and let $v_{\textnormal{max}} = \max_j |v_j|$. 
    Then the difference between the hard and soft attention outputs is
    \[ \left| \sum_{j=1}^n [\hardmax \vec{s} - \softmax \vec{s}]_j v_j \right| \le 4e^{-\tlscale} v_{\textnormal{max}}.\]
\end{lemma}
\begin{proof}
See \cref{sec:softmax_bound_constant_proof}.
\end{proof}

\subsection{Table Lookup}

Some constructions of hard-attention transformers use a \emph{table lookup} operation \citep{barcelo2023logical,strobl2024transformers}: given $\tablerow_i \in [n]$ for $i \in [n]$, we want to retrieve the value $v_{\tablerow_i}$.
This is accomplished using a parabola-shaped pattern that peaks at position $\tablerow_i$ (\cref{fig:approximations}).
Under hard attention, each position $i$ attends solely to position $j$ such that $j = \tablerow_i$.
To replicate this using soft attention, we must approximate it.
When $|v_j|$ is bounded, \cref{lem:softmax_bound_constant} suffices to bound the approximation error, but in one case, we have $v_j = j$, so we need the following lemma. 
\begin{lemma}
    \label{table-lookup-approx}
    Let $\tlscale \ge 1$ and let $\tablerow$ and $n$ be integers such that $1 \le \tablerow \le n$. For $j \in[n]$, define $s_{j} =\tlscale (2 \tablerow j - j^2)$.
    Then
    \[
    \left| \sum_{j=1}^n [\hardmax \vec{s} -\softmax \vec{s}]_j j \right| \le \tfrac32e^{-\tlscale}.
    \]
\end{lemma}

\begin{proof}
See \cref{sec:table-lookup-approx-proof}.
\end{proof}

\subsection{Tie-Breaking}

Next, we show some special cases of \cref{lem:softmax_bound_constant}, which are useful when using soft attention to simulate unique-hard attention. In unique-hard attention, if two positions are tied for the maximum score, the tie needs to be broken to the left or the right. 
To simulate this using soft attention, we need to %
add a \emph{tie-breaking} term, %
allowing $\softmax$ to approximate $\rhardmax$ or $\lhardmax$ sufficiently closely.

There are three variations, depending on whether the attention to be approximated is $\rhardmax$ or $\lhardmax$ and on whether the sequence length $n$ is known or not (\cref{fig:approximations}).

\begin{lemma} \label{thm:tie}
Let $s_1, \ldots, s_n$ be attention scores with gap $\gamma \ge 1$. Let $v_1, \ldots, v_n$ be attention values, and let $v_{\textnormal{max}} = \max_j v_j$.
\begin{enumerate}[label=(\alph*),ref=\alph*]
\item\label{thm:tie_causal} If\/ $\hat{s}_j = \gamma n^2 \left(s_j - \frac{1}{j}\right)$,
then
\[ \!\!\left|\sum_{j=1}^n [\rhardmax \vec{s} - \softmax \hat{\vec{s}}]_j v_j\right| \le 4e^{-\gamma} v_{\textnormal{max}}.\]
\item\label{thm:tie_rightmost} If\/ $\hat{s}_j = 2 \gamma n \left(s_j + \frac{j}{2n} \right)$,
then 
\[ \!\!\left|\sum_{j=1}^n [\rhardmax \vec{s} - \softmax \hat{\vec{s}}]_j v_j\right| \le 4e^{-\gamma} v_{\textnormal{max}}.\]
\item\label{thm:tie_leftmost} If\/ $\hat{s}_j = 2 \gamma n \left(s_j - \frac{j}{2n} \right)$,
then 
\[ \!\!\left|\sum_{j=1}^n [\lhardmax \vec{s} - \softmax \hat{\vec{s}}]_j v_j\right| \le 4e^{-\gamma} v_{\textnormal{max}}.\]
\end{enumerate}
\end{lemma}

\begin{proof}
By \cref{lem:softmax_bound_constant}; see \cref{sec:tie_proof}.
\end{proof}

\subsection{Approximate Booleans}\label{sec:approx_bool}

The output of hard attention is often interpreted as a Boolean value: $0$ for false and $1$ for true. If we approximate hard attention using soft attention, we get an approximate Boolean value. For any $\delta < 1/2$, we treat any number in $(-\infty,\delta]$ as false and any number in $[1-\delta,\infty)$ as true. For concreteness, we set $\delta = 1/4$. A transformer can convert approximate Booleans into exact ones using a FFN that rounds them to $0$ or~$1$.
\begin{lemma}
\label{lem:piecewise-simulation-ReLU-FFN}
    We can compute the following function with a FFN. 

    \[f(x) = \begin{cases}
        0 & x\leq \frac{1}{4}\\
        2x - \frac{1}{2} &  \frac{1}{4} \leq x \leq \frac{3}{4}\\
        1 & \frac{3}{4} \leq x.
    \end{cases}
    \quad
    \begin{tikzpicture}[baseline=0.5cm]
    \draw[very thick] (-0.5,0) -- (0.25,0) -- (0.75,1) -- (1.5,1);
    \begin{scope}[gray,every node/.style={font={\scriptsize}}]
    \draw (-0.5,0) -- (1.5,0) node[right] {$0$};
    \draw (-0.5,1) -- (1.5,1) node[right] {$1$};
    \draw (0,1.25) -- (0,-0.25) node[below] {$0$};
    \draw (1,1.25) -- (1,-0.25) node[below] {$1$};
    \end{scope}
    \end{tikzpicture}
    \]
    
\end{lemma}

\begin{proof}
    Any function that is continuous piecewise affine with a finite number of pieces can be computed by a FFN. Function $f$ can be computed by
        \[f(x) = 2\,\ReLU(x - \tfrac14) - 2\,\ReLU(x - \tfrac34). \tag*{\qedhere} \]
\end{proof}

\subsection{First and Last Position}

Finally, we show how to approximately mark the first and last positions of the input sequence and then correct the marker to an exact Boolean value.
\begin{proposition} \label{thm:first_last}
    A $\SMAT$ with a PE including $(-1)^i$ can compute the following functions $f \colon \Sigma^+ \lpto [0,1]^+$:
    \begin{subtheorems}
    \item\label{thm:first} $f(\str{w}) = (1, 0, \ldots, 0, 0)$ using future-masking
    \item\label{thm:last} $f(\str{w}) = (0, 0, \ldots, 0, 1)$ using past-masking
    \item\label{thm:one_over_i} $f(\str{w}) = (1/1, 1/2, \ldots, 1/(n-1), 1/n)$ using future-masking
    \item\label{thm:one_over_n_minus_i} $f(\str{w}) = (1/n, 1/(n-1), \ldots, 1/2, 1/1)$ using past-masking.
    \end{subtheorems}
\end{proposition}
\begin{proof}
To mark the first position (\ref{thm:first}), use uniform future-masked attention with values
$v_j = (-1)^{j+1}$.
If $i = 1$, then all attention is on position $1$, and the output value is $1$. If $i > 1$, then the attention output is at most $1/3$. So a FFN, similar to the one in \cref{lem:piecewise-simulation-ReLU-FFN}, can map all such values to $0$.
To mark the last position (\ref{thm:last}), we proceed as above, but using past-masked attention. If $i = n$, then the attention output is $1$ or $-1$; if $i < n$, the attention output is in $[-1/3,1/3]$. The FFN outputs $1$ if the former is true and $0$ if the latter is true.

Furthermore, we can compute the quantity $1/i$~(\ref{thm:one_over_i}) using uniform future-masked attention on the output of the above construction, and $1/(n+1-i)$~(\ref{thm:one_over_n_minus_i}) using past-masking.
\end{proof}

\section{Soft Attention, LTL and Counting}
\label{sec:soft-attention-LTLCmon}

As mentioned above, \citet{barcelo2023logical} introduced a temporal logic with counting, $\LTLCMon$ (to be defined below), and showed that it can be simulated by $\AHAT$s. 
The languages definable in $\LTLCMon$ therefore form a subclass of those recognized by $\AHAT$s that might be more amenable to simulation by $\SMAT$s.
To show this, we find it useful to factor $\LTLCMon$ into three logics, which account for three different types of hard-attention patterns:
\begin{itemize}
    \item $\TLPN$, which accounts for immediate predecessor and successor, for example, $n$-grams (\cref{sec:prev-and-next}).
    \item $\TLSU$, which accounts for tie-breaking, for example, flip-flops or induction heads (\cref{sec:since_until}).
    \item $\TLCCMon$, which accounts for numerical relations, for example, \prob{Parity} or \prob{Majority} (\cref{sec:utlcmon}).
\end{itemize}

In describing the first simulation in \cref{sec:prev-and-next}, we will make some general remarks that will apply to the other simulations as well.

\subsection{Previous and Next}
\label{sec:prev-and-next}

First we show that $\SMAT$s can simulate formulas in $\TLPN$, which has just the previous ($\prv$) and next ($\nxt$) operations.

This logic, like all the logics presented in this section, is a temporal logic, in which each formula $\phi$ is true or false depending on the position in a string. 
\begin{definition}[\TLPN]\label{def:tlpn}
    Let $\sigma$ denote symbols of a given alphabet $\Sigma$. The syntax of \TLPN{} is:
    \begin{align*}
        \phi &::= Q_\sigma \mid \lnot \phi \mid \phi_1\land\phi_2 \mid \phi_1\lor\phi_2 \mid \prv\phi\mid \nxt\phi
    \end{align*}
    The semantics of $\TLPN$ is given by the relation $\str{w}, i \models \phi$, which means that $\phi$ is true at position $i$ of $\str{w}$. This relation is defined as follows:
    \begin{align*}
        \str{w}, i  &\models Q_\sigma && \text{iff $w_i=\sigma$}\\
        \str{w}, i  &\models \lnot\phi && \text{iff $\str{w},i\not\models\phi$}\\
        \str{w}, i  &\models \phi_1\land\phi_2 && \text{iff $\str{w},i\models\phi_1$ and $\str{w},i\models\phi_2$}\\
        \str{w}, i  &\models \phi_1\lor\phi_2 && \text{iff $\str{w},i\models\phi_1$ or $\str{w},i\models\phi_2$}\\
        \str{w}, i  &\models \prv \phi && \text{iff $\str{w},(i-1)\models \phi$}\\
        \str{w}, i  &\models \nxt \phi && \text{iff $\str{w},(i+1)\models \phi$}
    \end{align*}    
A formula $\phi$ is \emph{satisfied} by a string $\str{w}$ iff $\phi$ is true at the last symbol of $\str{w}$, that is, $\str{w}, |\str{w}| \models \phi$.

\end{definition}

This logic accounts for the hard-attention patterns used in simulating $n$-gram models. For instance, the following $\TLPN$ formula detects the $3$-gram $101$ ending at the current position:
\begin{align*}
    \phi_{101} &= \prv\prv Q_1\land \prv Q_0\land Q_1
\end{align*}
Then $10101, 3 \models \phi_{101}$ and $10101, 5 \models \phi_{101}$, but $10101, 4 \not\models \phi_{101}$. Also, the whole string $10101$ satisfies $\phi_{101}$, that is, $10101 \models \phi_{101}$.

We now show how to simulate $\TLPN$ in a $\SMAT$. 
There are many conceivable ways to do this. We focus on $\SMAT$s that use temperature scaling (and bounded PEs) and those that use unbounded PEs (and no temperature scaling). In both cases, we also consider $\SMAT$s that use only future-masking and do not depend on $n$, suitable for autoregressive language modeling.

\begin{theorem}\label{thm:tlpn_to_smat}
Any formula $\phi$ of\/ $\TLPN$ can be simulated by a $\tauSMAT$ with either:
\begin{subtheorems}
    \item\label{thm:tlpn_to_smat_temp} $\tau(n) = 1/n$ and PEs including $i/n$ and $(-1)^i$
    \item\label{thm:tlpn_to_smat_pe} $\tau(n) = 1$ and PEs including $i/n$, $(-1)^i$, and $n$.
\end{subtheorems}
\end{theorem}

\begin{proof}

Given a formula $\phi$, we will construct a transformer that receives as input the string $\str{w}$, and in the final sequence of activation vectors, a designated coordinate gives the sequence of values $\ind{\str{w} , i \models \phi}$ for $i \in [|\str{w}|]$.

We construct this transformer by induction on the structure of a formula.
Each subformula's truth value is stored in a different component of the activation vectors.
For the base case $\phi = Q_\sigma$, we store a $1$ in the corresponding component of the word embedding $\textnormal{WE}(\sigma)$, as shown by \citet{angluin2023masked}.
For the inductive step, where $\phi$ is formed out of smaller subformulas $\phi_1$ and possibly $\phi_2$, 
we assume that the results of the computation of $\phi_1$ and $\phi_2$ are available in two components of the activation vector, and we construct a new layer that computes $\phi$ from $\phi_1$ and $\phi_2$. The new value is placed in a previously unused coordinate, with all other coordinates zero, and the residual connection supplies all of the previously computed outputs.

Boolean operations can be computed by a FFN as shown by \citet{angluin2023masked}.
The remaining cases are the operators $\prv \phi_1$ and $\nxt \phi_1$.
Let $v_1, \ldots, v_n$ be the values of $\phi_1$.

In a $\UHAT$, we could compute $\prv \phi_1$ at positions $i>1$ by making each position $i$ attend only to position $(i-1)$. If $i$ is odd (and greater than $1$), we would make the attention scores greater for even positions, so that rightmost tie-breaking would select position $i-1$. If $i$ is even, we would make the attention scores greater for odd positions.
If $i=1$, we would output $0$.
But in a $\tauSMAT$, we have to approximate this.

Condition (\ref{thm:tlpn_to_smat_temp}): 
We use two future-masked attention layers (or one layer with two heads).
In the first attention, the queries and keys are set so that the attention scores are
        \[\hat{s}_{ij} = 6\left(\tfrac12 (-1)^{j} + \tfrac{j}{2n}\right)\]
which is maximized at the rightmost even position $j \le i$.  Applying \cref{thm:tie}\ref{thm:tie_rightmost} (with $s_{ij} = \tfrac12 (-1)^{j}$), we get %
\[ \left|\sum_j [\rhardmax \vec{s} - \softmax \hat{\vec{s}}]_j v_j\right| \le 4e^{-3} < \tfrac14.\]

Similarly, the second attention attends to the rightmost odd position $j \le i$, using scores
        \[\hat{s}_{ij} = 6\left(\tfrac12 (-1)^{j+1} + \tfrac{j}{2n}\right).\]
The FFN, at each position $i$, tests whether $i$ is odd or even; if odd, it uses the output of the first attention, and if even, it uses the output of the second attention. 
Position $i=1$ may be set to $0$ with a FFN using \cref{thm:first_last}\ref{thm:first}.

Similarly, to compute $\nxt \phi_1$, we use two past-masked attention layers with scores as above, but with the tie-breaking terms ${}+j/(2n)$ changed to ${}-j/(2n)$,
and bound the error using \cref{thm:tie}\ref{thm:tie_leftmost}. Position $i=n$ may be set to $0$ with a FFN using \cref{thm:first_last}\ref{thm:last}.

Condition (\ref{thm:tlpn_to_smat_pe}): Same as condition~(\ref{thm:tlpn_to_smat_temp}), but scale the query vectors by a factor of $n$.
\end{proof}

\begin{theorem}
\label{thm:thm:tlpn_to_smat_no_n}
Any formula $\phi$ of\/ $\TLP$ can be simulated by a future-masked $\tauSMAT$ with:
\begin{subtheorems}[resume]
    \item $\tau(i) = 1/i^2$ and PEs containing $(-1)^i$
    \item $\tau(i) = 1$ and PEs containing $(-1)^i$ and $i^2$.
\end{subtheorems}
\end{theorem}
\begin{proof}
Same as the proof of \cref{thm:tlpn_to_smat}, but use \cref{thm:tie}\ref{thm:tie_causal}. 
The quantity $1/i$ required by \cref{thm:tie}\ref{thm:tie_causal} is supplied by \cref{thm:first_last}\ref{thm:one_over_i}.
\end{proof}

\subsection{Since and Until}
\label{sec:since_until}

Next, we show that $\SMAT$s can compute the values of $\since$ and $\until$, which we now define.

\begin{definition}[\TLSU]
    Let $\sigma\in\Sigma$ denote symbols of a given alphabet $\Sigma$. The temporal logic \TLSU{} has syntax
     \begin{align*}
        \phi &::= \begin{aligned}[t] &Q_\sigma \mid \lnot \phi \mid \phi_1\land\phi_2 \mid \phi_1\lor\phi_2\\
        & \mid \phi_1\since\phi_2\mid \phi_1\until\phi_2
        \end{aligned}
    \end{align*}
    The semantics of Boolean expressions is as in \cref{def:tlpn}. 
    The intuitive meaning of $\phi_1 \since \phi_2$ is ``since the last time $\phi_2$ was true, $\phi_1$ has been true,'' and $\phi_1 \until \phi_2$ means ``$\phi_1$ will be true until the next time $\phi_2$ is true.'' More formally:
    \begin{align*}
    \str{w}, i &\models \phi_1\since\phi_2 && \!\text{iff there is $j\leq i$ s.t. $\str{w},j\models \phi_2$}\\
    &&&\text{and for all $k$ s.t. $j\leq k\leq i$}\\
    &&&\text{we have $\str{w},k\models \phi_1$}\\
    \str{w}, i  &\models \phi_1\until\phi_2 && \!\text{iff there is $j\geq i$ s.t. $\str{w},j\models \phi_2$}\\
    &&&\text{and for all $k$ s.t. $i\leq k\leq j$}\\
    &&&\text{we have $\str{w},k\models \phi_1$.}
    \end{align*}
\end{definition}

\begin{theorem}\label{thm:tlsu_to_smat}
    Any formula $\phi$ of\/ $\TLSU$ can be simulated by a $\tauSMAT$ with either:
    \begin{subtheorems}
        \item\label{item:ltl_temponly} $\tau(n) = 1/n$ and PEs including $i/n$
        \item\label{item:ltl_posonly} $\tau(n)=1$ and PEs including $i/n$ and $n$.
    \end{subtheorems}
\end{theorem}

    \begin{proof}
        By induction on the structure of $\phi$. The $Q_\sigma$ and Boolean operation cases are as in \cref{thm:tlpn_to_smat}. 

        The case $\phi = \phi_1 \since \phi_2$ is intuitively computed as follows. We look at positions $j = i, i-1, i-2$, and so on. As long as $\phi_1$ is true and $\phi_2$ is false, we keep looking to the left. But we stop when we encounter one of these situations (black means true, blank means false, and gray means either):
  \begin{center}
    \begin{tabular}{c@{\qquad}c@{\qquad}c}
      \begin{tikzpicture}[x=3mm]
        \foreach \i in {0,1,...,7} { \draw[gray] (\i,1.5) -- (\i,-0.5); };
        \node[anchor=north] at (1.5,-0.5) {$j$};
        \node[anchor=north] at (5.5,-0.5) {$i$};
        \draw (0,1) node[anchor=east] {$\phi_2$} -- (7,1);
        \draw (0,0) node[anchor=east] {$\phi_1$} -- (7,0);
        \begin{scope}[line width=1mm]
          \draw[gray!50] (0,1) -- (1,1);
          \draw (1,1) -- (2,1);
          \draw[gray!50] (6,1) -- (7,1);
          \draw[gray!50] (0,0) -- (1,0);
          \draw (1,0) -- (6,0);
          \draw[gray!50] (6,0) -- (7,0);
        \end{scope}
      \end{tikzpicture} &
      \begin{tikzpicture}[x=3mm]
        \foreach \i in {0,1,...,7} { \draw[gray] (\i,1.5) -- (\i,-0.5); };
        \node[anchor=north] at (1.5,-0.5) {$j$};
        \node[anchor=north] at (5.5,-0.5) {$i$};
        \draw (0,1) node[anchor=east] {$\phi_2$} -- (7,1);
        \draw (0,0) node[anchor=east] {$\phi_1$} -- (7,0);
        \begin{scope}[line width=1mm]
          \draw[gray!50] (0,1) -- (2,1);
          \draw[gray!50] (6,1) -- (7,1);
          \draw[gray!50] (0,0) -- (1,0);
          \draw[gray!50] (6,0) -- (7,0);
          \draw (2,0) -- (6,0);
        \end{scope}
      \end{tikzpicture} 
      \\
      $\str{w}, i \models \phi_1 \since \phi_2$ & $\str{w}, i \not\models \phi_1 \since \phi_2$ %
    \end{tabular}
  \end{center}

  On the left, $j$ is the rightmost position satisfying $\phi_2$, and $\phi_1$ is true at positions $j$ to $i$, so $w,i\models \phi_1 \since \phi_2$. On the right, the rightmost position satisfying $\phi_2$ is at or to the left of $j$, but $\phi_1$ is false at $j$, so $w,i\not\models \phi_1 \since \phi_2$. If we never encounter either case, there is no $j$ satisfying $\phi_2$, so $w,i\not\models \phi_1 \since \phi_2$.

  Thus, in a $\UHAT$ \citep{angluin2023masked}, we could attend to the rightmost position $j$ maximizing $\ind{\str{w},j\models\lnot \phi_1\lor \phi_2}$ and return the value of $\phi_1 \land \phi_2$ at $j$. But in a $\tauSMAT$, we have to approximate this.

        Under condition (\ref{item:ltl_temponly}), add a future-masked self-attention layer with:
        \begin{align*}
            \hat{s}_{ij} &= 6 \left(\ind{\str{w}, j \models \neg\phi_1\lor \phi_2}+\tfrac{j}{2n}\right) \\
            v_j &= \begin{bmatrix} \ind{\str{w}, j \models \phi_1 \land \phi_2} \end{bmatrix}.
        \end{align*}
        Using temperature $\tau \le 1/n$ satisfies the requirements of \cref{thm:tie}\ref{thm:tie_rightmost} (with $s_{ij} = \ind{\str{w}, j \models \neg\phi_1\lor \phi_2}$), so we have  
        \[ \left|\sum_j [\rhardmax \vec{s} - \softmax \hat{\vec{s}}]_j v_j\right| \le 4e^{-3} < \tfrac14.\]
        This yields approximate Boolean values for $\phi_1\since \phi_2$, which are rounded to $0$ or $1$ using a FFN (\cref{lem:piecewise-simulation-ReLU-FFN}).
        
        The simulation of $\phi_1\until\phi_2$ is symmetric, using the tiebreaker of $-j/(2n)$ and \cref{thm:tie}\ref{thm:tie_leftmost}. Under condition (\ref{item:ltl_posonly}), we scale the queries $\qvec{i}$ by a factor of $n$.
    \end{proof}

For decoder-only models (with obligatory future-masking) we also describe a simulation of $\TLS$ that does not depend on~$n$.
\begin{theorem}\label{thm:since_to_smat}
    Any formula $\phi$ of\/ $\TLS$ can be simulated by a future-masked $\tauSMAT$ with PEs including $1/i$ and either: 
    \begin{subtheorems}
        \item\label{item:ptl_temponly} $\tau(i) = 1/i^2$ and PEs including $1/i$
        \item\label{item:ptl_posonly} $\tau(i) = 1$ and PEs including $1/i$ and $i^2$.
    \end{subtheorems}
\end{theorem}

\begin{proof}
To simulate $\since$, use a future-masked attention layer with:
\begin{align*}
\hat{s}_{ij} &= 3 \left( \ind{\str{w}, j \models \neg\phi_1\lor \phi_2}-\tfrac{1}{j} \right) \\
v_j &= \begin{bmatrix} \ind{\str{w}, j \models \phi_1\land\phi_2} \end{bmatrix}.
\end{align*}
Using either a temperature of $1/i^2$ or multiplying $\qvec{i}$ by $i^2$ satisfies the requirements of \cref{thm:tie}\ref{thm:tie_causal} (with $s_{ij} = \ind{\str{w}, j \models \neg\phi_1\lor \phi_2}$).
The result then follows after rounding the approximate Boolean values to $0$ or~$1$.
\end{proof}  

\subsection{$\LTL$}
\label{sec:ltl}

    Combining $\TLPN$ and $\TLSU$ we get $\TLPNSU=\LTL$. Masked $\UHAT$s recognize exactly those formal languages definable by $\LTL$ \citep{angluin2023masked}. %
    It follows from the above that $\LTL$, and therefore any masked $\UHAT$, can be simulated by a $\tauSMAT$.
\begin{corollary}\label{cor:ltl_to_smat}
    Any formula $\phi$ of\/ $\LTL$ can be simulated by a $\tauSMAT$ with either:
    \begin{subtheorems}
        \item\label{cor:ltl_to_smat_bounded} $\tau(n)=1/n$ and PEs including $i/n$ and $(-1)^i$
        \item $\tau(n)=1$ and PEs including $i/n$, $(-1)^i$, and $n$.
    \end{subtheorems}
\end{corollary}
\begin{proof}
By \cref{thm:tlpn_to_smat,thm:tlsu_to_smat}.
\end{proof}

We can also simulate the past fragment of $\LTL$  (with only $\since$ and $\prv$) using only future-masking, suitable for a decoder-only transformer.

\begin{corollary}\label{cor:ltl_since_to_smat}
    Any formula of\/ $\TLPS{}$ can be simulated by a future-masked $\tauSMAT$ with
    \begin{subtheorems}
    \item $\tau(i)=1/i^2$ and PEs including $(-1)^i$
    \item $\tau(i)=1$ and PEs including $(-1)^i$ and $i^2$.
    \end{subtheorems}
\end{corollary}

\begin{proof}
By \cref{thm:thm:tlpn_to_smat_no_n,thm:since_to_smat}.
The quantity $1/i$ required by \cref{thm:since_to_smat} is supplied by \cref{thm:first_last}\ref{thm:one_over_i}.
\end{proof}

We recall that \citet[Lemma~5]{hahn-2020-theoretical}  places a bound $O(1/n)$ on the gap between the output for acceptance or rejection. For a fixed-size gap, some activation in the transformer must be scaled up by a factor of~$n$. The constructions in \cref{thm:tlpn_to_smat}, \cref{thm:tlsu_to_smat}, and \cref{cor:ltl_to_smat} use $\tau(n) = 1/n$ or scale queries by $n$, and are optimal in the sense that we should not expect the factor of $n$ to be asymptotically any smaller.

\subsection{Counting}\label{sec:utlcmon}

As the third step, we show that $\tauSMAT$ can simulate a logic 
with counting terms and 
numerical predicates that can be applied to both positions and counts \citep{barcelo2023logical}. 

\begin{definition}\label{def:TLCCMon}
    A \emph{unary numerical predicate} is a family of functions $\theta=(\theta_n)_{n>0}$ where $\theta_n\colon \{0,1, \ldots, n\}\to\{0,1\}$. We write $\Mon$ for the set of all unary numerical predicates. %
    The logic \TLCCMon{} has syntax
        \begin{align*}
            \phi &::= \begin{aligned}[t] &Q_\sigma \mid \lnot \phi \mid \phi_1\land\phi_2
            \mid \phi_1\lor\phi_2 \\
                & \mid \theta \mid \theta(\kappa) \mid t_1 < t_2\\
            \end{aligned} \\
            t &::= \kappa \mid t_1+t_2 
            \mid t_1-t_2 
            \mid 1 \\
            \kappa &::= \countl[\phi] \mid\countr[\phi]
        \end{align*}
    Formulas are interpreted as follows:
    \begin{align*}
        \str{w},i&\models \theta && \text{iff $\theta_{|w|}(i)$}\\
        \str{w},i&\models \theta(\kappa) && \text{iff $\theta_{|w|}(\kappa^{\str{w},i})$}\\
        \str{w},i&\models t_1 < t_2 && \text{iff $t_1^{(\str{w},i)} < t_2^{(\str{w},i)}$.}
    \end{align*}
    Terms are interpreted as integers:
    \begin{align*}
        \countl[\phi]^{(\str{w},i)} & = |\{j\leq i\mid \str{w},j\models\phi\}|\\
        \countr[\phi]^{(\str{w},i)} & = |\{j\geq i\mid \str{w},j\models\phi\}|\\
        (t_1+t_2)^{(\str{w},i)} & = t_1^{(\str{w},i)}+t_2^{(\str{w},i)}\\
        (t_1-t_2)^{(\str{w},i)} & = t_1^{(\str{w},i)}-t_2^{(\str{w},i)}\\
        1^{(\str{w},i)} &= 1.
    \end{align*}
    \end{definition}
    For example, \prob{PARITY} is the set of binary strings with an odd number of $1$s.  
    Using the numerical predicate $\textsf{ODD}$, where $\textsf{ODD}_n(i)$ is true iff $i$ is odd, we can define this using the following $\TLCCMon$ formula:
    \begin{align*}
    \phi_{\prob{Parity}}
    &= \textsf{ODD}\mleft( \countl \mleft[Q_1\mright]\mright).
    \end{align*}
    When evaluated at the last position, $\countl[Q_1]$ is the number of occurrences of $1$ in the whole string, and $\phi_{\prob{Parity}}$ tests whether this number is odd.
   
    We show how $\SMAT$s can compute formulas of $\TLCCMon$. %

    \begin{theorem} \label{thm:tlccmon_to_smat}
    Any formula $\phi$ of\/ $\TLCCMon$ can be simulated by a $\tauSMAT$ with either:
    \begin{subtheorems}
        \item\label{tlccmon_to_smat_temp} $\tau(n)=1/n$ and PEs including $1/i$, $1/(n-i+1)$, $i$, $i^2$, and all numerical predicates in $\phi$
        \item\label{tlccmon_to_smat_pe} $\tau(n)=1$ and PEs including $1/i$, $1/(n-i+1)$, $i$, $in$, $i^2n$, and all numerical predicates in $\phi$.
    \end{subtheorems}
    \end{theorem}

    \begin{proof}
    We show condition (\ref{tlccmon_to_smat_temp}) in detail and remark briefly at the end how to modify the proof for condition (\ref{tlccmon_to_smat_pe}).

        The proof is again by induction on the structure of $\phi$. The cases we need to consider are:
        \begin{itemize}
        \item $\phi = Q_\sigma$ or $\phi$ is a Boolean operation: As in \cref{thm:tlpn_to_smat}.
        \item $\phi = \theta$: We store $\theta_n(i)$ in the corresponding component of the position embedding $\textnormal{PE}_n(i)$.
        \item $\phi = (\countl[\phi_1] > 0)$ or $(\countr[\phi_1] > 0)$: This case is proved below, both as a warm-up for the other cases as well as a subroutine of the other cases.
        \item $\phi = (t_1 < t_2)$: Proved below.
        \item $\phi = \theta(\countl[\phi_1])$ or $\theta(\countr[\phi_1])$: Proved below.
        \end{itemize}

        \paragraph{Case $\phi = (\countl[\phi_1] > 0)$:} 
        We can use an FFN to test whether $\phi_1$ is true at the current position $i$, and if so, we know that $\countl[\phi_1]>0$, so we return $1$. Otherwise, we know that $\countl[\phi_1]<i$, as we assume below.
        
        The essential idea is simply to use uniform attention to perform the counting. However, attention does not count; it averages. So the challenge is to turn averages into counts.
        The construction proceeds in three steps. 
        First, we compute $\countl[\phi_1]/i$ exactly.
        Second, we scale up to $\countl[\phi_1]$, but only approximately.
        Third, we test whether the count is greater than $0$.
        
        \subparagraph{Step 1: Average.} The first step of computing $\countl[\phi_1]/i$ can be done using masked uniform attention \citep{yang2024counting}. 
        
        \subparagraph{Step 2: Scale.} For the second step, we approximately retrieve $\bigl(\countl[\phi_1]+1\bigr)$ using a table-lookup attention layer with 
        \begin{align*}
            \qvec{i}&=3 \begin{bmatrix} \frac{\countl[\phi_1]}{i} + \frac1i \\[1ex] 
            \frac{1}{i}\end{bmatrix} &
            \kvec{j}&=\begin{bmatrix} 2j\\ -j^2 \end{bmatrix} &
            v_j&=\begin{bmatrix} j \end{bmatrix}.
        \end{align*}
        Then $s_{ij} = 3\bigl(2\bigl(\countl[\phi_1]+1\bigr)j-j^2\bigr)/i$,
        which is uniquely maximized at $j = \countl[\phi_1]+1$. 
        We apply \cref{table-lookup-approx} with temperature $\tau=1/n$ and gap $\tlscale = 3/(i\tau)$ to obtain
        \begin{gather*}
        c_i = \sum_j [\softmax_\tau(s_{i1}, \ldots, s_{ii})]_j v_j \\
        \left|\,c_i - \bigl(\countl[\phi_1]+1\bigr)\,\right| \le
        4e^{-3n/i}
        \leq 4e^{-3}
        < \frac{1}{4}.
        \end{gather*}
        Thus, the output value $c_i$ is within $\countl[\phi_1]+1\pm1/4$.

        \subparagraph{Step 3: Compare.} The third step is to apply the FFN of \cref{lem:piecewise-simulation-ReLU-FFN} to $(c_i-1)$, which will output~$0$ if $\countl[\phi_1] = 0$ and $1$ if $\countl[\phi_1] > 0$.

        \paragraph{Case $\phi = (t_1 < t_2)$:}If $\phi$ is a comparison of count terms, it can be rewritten in the following form. Let $\Kleft$ and $\Kright$ be disjoint sets of indices, and let $C$ and $\lambda_k$ (for $k \in \Kleft \cup \Kright$) be integers:
        \[
            \phi = \left(\displaystyle\sum_{k\in \Kleft} \lambda_k \countl[\phi_k] + \sum_{k\in \Kright} \lambda_k \countr[\phi_k] > C\right).
        \] 
        Let $\Lambda=\sum_{k\in \Kleft \cup \Kright }|\lambda_k|$. If $\Lambda = 0$, the formula is constant, so we assume without loss of generality that $\Lambda \ge 1$.

        For each $k\in \Kleft$, 
        we need to approximate $\countl[\phi_k]$.
        (Terms $\countr[\phi_k]$ for $k \in \Kright$ are similar and described briefly afterwards.)
        We can use the construction in the previous case to test whether $\countl[\phi_k]=0$; if not, then we approximate $\countl[\phi_k]$ as follows, assuming that $\countl[\phi_k]>0$.
        
        \subparagraph{Step 1: Average.} The first step of computing $\countl[\phi_k]/i$ for each $k$ is the same as in the previous case.
        
        \subparagraph{Step 2: Scale.} For the second step, we approximately retrieve $\countl[\phi_k]$ using a table-lookup attention layer with 
        \begin{align*}
            \qvec{i}&=3\Lambda \begin{bmatrix} \frac{\countl[\phi_k]}{i} \\[1ex] 
            \frac{1}{i}\end{bmatrix} &
            \kvec{j}&=\begin{bmatrix} 2j\\ -j^2 \end{bmatrix} &
            v_j&=\begin{bmatrix} j \end{bmatrix}.
        \end{align*}
        Then $s_{ij} = 3\Lambda(2\countl[\phi_k]j-j^2)/i$,
        which is uniquely maximized at $j = \countl[\phi_k]$. We apply \cref{table-lookup-approx} with temperature $\tau=1/n$ and gap $\tlscale = 3\Lambda/(i\tau)$ to obtain an approximate count
                \[c_k(i) = \sum_j [\softmax_\tau(s_{i1}, \ldots, s_{ii})]_j v_j\]
        that is within $\countl[\phi_k]\pm1/(4\Lambda)$, which is close enough to $\countl[\phi_k]$ for the comparison step.
        
        Computing a right-counting term $\countr[\phi_k]$ requires only a slight modification of the second step. The first step gives $\countr[\phi_k]/(n-i+1)$.
        The second step requires the quantity $1/(n-i+1)$.
        Since $n-i+1\leq n$, the bound of $1/(4\Lambda)$ is derived in the same way.
        
        \subparagraph{Step 3: Compare.} The third step is to compute the linear combination using a FFN:
        \begin{align*}
          H(i) &= \sum_{\mathclap{k \in \Kleft \cup \Kright}} \lambda_k c_k(i).
        \end{align*}
        The error in $H(i)$ is
        \begin{align*}
        \firstline{\left| \, H(i) - \left(\sum_{k \in \Kleft} \lambda_k \countl[\phi_k] + \sum_{k \in \Kright} \lambda_k \countr[\phi_k] \right) \, \right|} \\
        &\leq \left|\sum_{k \in \Kleft \cup \Kright} \mkern-8mu \lambda_k\frac{1}{4\Lambda} \right|\leq\frac{1}{4}.
        \end{align*}
        If $H(i)-C \geq 3/4$ then $\phi$ is true; if $H(i)-C \leq 1/4$ then $\phi$ is false.
        So we can compute the correct Boolean value using \cref{lem:piecewise-simulation-ReLU-FFN}. 

        \paragraph{Case $\phi = \theta(\countl[\phi_1])$:} If $\phi$ is a numerical predicate $\theta$ applied to a count, 
        we test whether $\countl[\phi_1] = 0$ as above; if so, return $\theta(0)$. Otherwise,
        we perform the first step as above. 
        We perform the second step using scores 
        $s_j = 3(2\countl[\phi_1]j-j^2)/i$, temperature $\tau=1/n$, and values $v_j=\theta_n(j)$.
        Since $|\theta_n(j)| \le 1$, we can use \cref{lem:softmax_bound_constant} with $\tlscale=3/(i\tau)$ to ensure that the attention output approximates $\theta_n(\countl[\phi_1])$ with error at most $4e^{-\gamma} \le 1/4$, which we can correct using a FFN (\cref{lem:piecewise-simulation-ReLU-FFN}).
        
        \paragraph{Case $\phi = \theta(\countr[\phi])$:}  This is just the mirror image of the previous case.

        \bigskip

        This completes the proof for condition (\ref{tlccmon_to_smat_temp}).         
        The proof for condition (\ref{tlccmon_to_smat_pe}) is identical, except that in the second step, we scale $\kvec{j}$ by $n$.
    \end{proof}

    \begin{theorem} \label{thm:tlcmon_to_smat_temp}
    Any formula $\phi$ of\/ $\TLCMon$ can be simulated by a future-masked $\tauSMAT$ with $\tau(i)=1/i$ and PEs including $1/i$, $i$, $i^2$, and all numerical predicates in $\phi$.
    \end{theorem}
    \begin{proof}
    Same as \cref{thm:tlccmon_to_smat}\ref{tlccmon_to_smat_temp}.
    \end{proof}

\subsection{$\LTLCMon$}\label{sec:ltlcmon}

The logic $\LTLCMon$ includes all previously discussed temporal operators ($\prv$,$\nxt$,$\since$,$\until$), counting terms $\countl$,$\countr$, linear inequalities of counting terms, and all unary numerical predicates. It is exactly the logic $\LTL(\mathbf{C},+)$  introduced by \citet{barcelo2023logical}.

For example, the following formula of $\LTLCMon$ recognizes \prob{parity} without using numerical predicates:
    \begin{align*}
    \phi_{\prob{Parity}} %
    &= \left( \countl \mleft[\countl[\prv Q_1] = \countr[Q_1]\mright] = 0 \right).
    \end{align*}
    The total number of $1$'s is odd iff there is no position $i$ where the number of $1$'s strictly to the left of $i$ is equal to the number of $1$'s at or to the right of $i$. Above, the subformula $\prv Q_1$ tests whether there is a $1$ immediately to the left, so $\countl[\prv Q_1]$ counts the number of $1$'s strictly to the left. 

\Citet{barcelo2023logical} simulate this logic using $\AHAT$s, but here we give conditions under which this logic can be simulated using $\tauSMAT$s.

\begin{corollary}
\label{cor:smat-for-LTLCMon}
    Any formula $\phi$ of\/ $\LTLCMon$ can be simulated by a $\tauSMAT$ with PEs including all numerical predicates in $\phi$ and either:
    \begin{subtheorems}
        \item\label{ltlcmon_to_smat_temp} $\tau(n)=1/n$ and PEs including $(-1)^i$, $i/n$, $i$, and $i^2$
        \item\label{ltlcmon_to_smat_pe} $\tau(n)=1$ and PEs including $(-1)^i$, $i/n$, $i$, $n$, $in$,  and $i^2n$.
    \end{subtheorems}
\end{corollary}

\begin{proof}
By \cref{cor:ltl_to_smat,thm:tlccmon_to_smat}. The quantities $1/i$ and $1/(n-i+1)$ required by \cref{thm:tlccmon_to_smat} are supplied by \cref{thm:first_last}\ref{thm:one_over_i}\ref{thm:one_over_n_minus_i}.
\end{proof}

\begin{corollary} \label{thm:TLPSCMon_to_SMAT}
    Any formula $\phi$ of\/ $\TLPSCMon$ can be simulated by a future-masked $\tauSMAT$ with $\tau(i)=1/i^2$ and PEs including $(-1)^i$, $i$, $i^2$, and all numerical predicates in $\phi$.
\end{corollary}

\begin{proof}
By \cref{cor:ltl_since_to_smat,thm:tlcmon_to_smat_temp}. The quantity $1/i$ required by \cref{thm:tlcmon_to_smat_temp} is supplied by \cref{thm:first_last}\ref{thm:one_over_i}.
\end{proof}

\section{Average-Hard Attention }
\label{sec:uuhat_to_smat_temp}

The simulations of temporal logics in the preceding section compute exact Boolean values for every subformula, because a FFN can round an approximate Boolean to a exact Boolean.
Counting terms are approximated sufficiently well that comparisons between them, and applications of numerical predicates to them, produce approximate Booleans.
But for the case of simulating an $\AHAT$, there is no similar option for eliminating approximation error, because in general the values in the computation are not confined to a finite set.

There is a simple argument that if we take an $\AHAT$ $T$ and replace its $\avghardmax$ weighting function with  $\softmax_\tau$ to produce a $\tauSMAT$ $\hat{T}$, then for any input sequence, as $\tau$ goes to zero, the limit of the sequence of activation vectors of the final layer of $\hat{T}$ equals the sequence of activation vectors of the final layer of $T$. 
This is because a transformer is the composition of continuous functions, and the $\softmax_\tau$ function approaches the $\avghardmax$ function as $\tau$ goes to zero.
This argument is likely the basis of many researchers' confidence that $\SMAT$s can approximate $\AHAT$s; however, it does not by itself bound the rate of convergence.
 
Here, we bound the rate of convergence for $\AHAT$s in terms of the gap separating the maximal attention score or scores from other scores. We say that an attention layer has \emph{gap $\gamma(n)$} iff for every input with length $n$ and every position $i$,
the attention scores at $i$ (that is, $s_{i,*}$) have gap $\gamma(n)$ (\cref{def:tieless}).
An $\AHAT$ has gap $\gamma(n)$ if all of its attention layers have gap $\gamma(n)$.
Existing uses of average-hard attention in the literature generally have gap $\Omega(1/n^k)$ for small constants $k$ (see \cref{sec:applications}).

\begin{proposition}
    \label{lem:existence-of-tieless-gap}
    Every $\AHAT$ has a strictly positive gap function.
\end{proposition}
\begin{proof}
For any $n$, because there are only finitely many possible input strings of length $n$, we may take $\gamma(n)$ to be the minimum absolute difference, over all input strings, attention layers, and positions $i$, between the maximal attention score $s_{\textnormal{max}} = \max_j s_{ij}$ and any other attention score $s_{ij'} < s_{\textnormal{max}}$.
\end{proof}

Let $T$ be an $\AHAT$ with gap $\gamma(n)$.
We transform $T$ into a $\tauSMAT$ $\hat{T}$ that has the same parameters as $T$ but uses temperature-scaled softmax attention in place of average-hard attention.
By choosing a temperature function $\tau(n)$ depending on $\gamma(n)$, we show that for every input, the output of $\hat{T}$ closely approximates the output of $T$.

We define the function $x_{\textnormal{max}}(n)$ to be the maximum of $1$ and the absolute value of any entry in $\vec{x}_i$ over all inputs of length $n$ and positions $i$, where $(\vec{x}_1,\ldots,\vec{x}_n)$ is the initial activation sequence for the input.  This is well defined because there are finitely many different inputs of length $n$ (similar to the proof of \cref{lem:existence-of-tieless-gap}).

First, we prove the following lemma bounding the error introduced by a single transformer layer.
\newcommand{\invec}{\vec{g}}
\newcommand{\outvec}{\vec{h}}
\renewcommand{\invec}{\vec{g}}
\renewcommand{\outvec}{\vec{h}}

\begin{lemma}
\label{lem:one-layer-error-tie}
    Let $T$ be an $\AHAT$ with gap $\gamma(n)$, and let $\hat{T}$ be the corresponding $\tauSMAT$ with temperature $\tau(n) \le 1$.  There exists a constant $K \ge 1$ such that for all positive integers $n$ and all sufficiently small $\epsilon > 0$, we have the following, for each layer $\ell$.
    For any $(\invec_1, \ldots, \invec_n)$ and $(\hat{\invec}_1,\ldots,\hat{\invec}_n)$ such that
    $\|\invec_i - \hat{\invec}_i\|_1 \le \epsilon$ for all $i$,
    let $(\outvec_1,\ldots,\outvec_n)$ be the output of layer $\ell$ of $T$ on input $(\invec_1,\ldots,\invec_n)$, and let $(\hat{\outvec}_1,\ldots,\hat{\outvec}_n)$ be the output of layer $\ell$ of $\hat{T}$ on input $(\hat{\invec}_1,\ldots,\hat{\invec}_n)$.
    Then, for all $i$,
\begin{equation*}
\norm{\outvec_i - \hat{\outvec}_i} \le K x_{\textnormal{max}}(n) \left( n e^{-\frac{\gamma(n)}{\tau(n)}} + \tfrac{x_{\textnormal{max}}(n) }{\tau(n)} \epsilon \right). \tag*{\qedhere}
\end{equation*}
\end{lemma}

\begin{proof}
    See \cref{sect:proof-lemma-one-layer-error}.
\end{proof}
The first term inside the parentheses bounds the error due to approximating $\avghardmax$ by $\softmax$, while the second term bounds the error propagated from the layer input to the layer output.

By iterating this lemma over the layers of an $\AHAT$, we prove the following theorem.
\begin{theorem}
\label{thm:smat_approx_hard_attention}
    Let $T$ be an $\AHAT$ with gap $\gamma(n)$.
    For any function $\epsilon(n) \in \Omega(1/\poly(n))$, there exists a temperature function $\tau(n)$ with
    \[\frac1{\tau(n)} \in O\left(\frac{1}{\gamma(n)} \log \frac{nx_{\textnormal{max}}(n)}{\gamma(n)} \right)\]
    such that the $\tauSMAT$
    $\hat{T}$ corresponding to $T$ approximates $T$ in the following sense.
    For every input, if  $(\vec{y}_1, \ldots, \vec{y}_n)$ is the output of $T$ and $(\hat{\vec{y}}_1, \ldots, \hat{\vec{y}}_n)$ is the output of $\hat{T}$, then \[ \norm{\vec{y}_i - \hat{\vec{y}}_i} \le \epsilon(n) \qquad \text{for all $i \in [n]$.} \]
\end{theorem}
In particular, if $\gamma(n) = 1/n^k$ and $x_{\textnormal{max}}(n) \in n^{O(1)}$, we have $1/\tau(n) \in O(n^k \log n)$. Note that the constant factor in the bound hides quantities not depending on $n$, including an exponential dependence on the depth $L$.
\begin{proof}
See \cref{sec:smat_approx_hard_attention_proof}.
\end{proof}

\section{Some Applications}
\label{sec:applications}

Below, we consider some discrete tasks that have been solved using $\UHAT$s or $\AHAT$s.

\paragraph{Induction heads.}

In the \emph{basic induction head task} \cite{elhage2021mathematical}, the input is a string $w_1 \cdots w_n$, and the goal is to find the rightmost $i<n$ such that $w_i = w_n$ and output $w_{i+1}$.
\citet{sanford2024transformersparallelcomputationlogarithmic} consider the more general \emph{$k$-hop induction head task}, which iterates the basic induction head task a constant $k$ number of times. 

The basic induction head task can be defined exactly in $\LTL$ \cite{angluin2023masked}, and the $k$-hop induction head task can be computed by iterating the basic induction head task $k$ times, so
\cref{cor:ltl_to_smat} gives a $\tauSMAT$ for the $k$-hop induction head task with $\tau(n) = 1/n$.
Previous solutions for these tasks include the following.

\Citet{liu2023exposing} show that the \emph{flip-flop} language, a particular instance of the basic induction head task, can be recognized exactly by an $\AHAT$
with gap $\gamma(n) = 1/n$ and $x_{\textnormal{max}}(n) \in O(1)$.
They also give a $\SMAT$ in which queries are scaled by $O(n \log n)$, which is equivalent to a temperature of $\Omega(1/(n \log n))$, a little lower than our $\tau(n) = 1/n$ solution.

\Citet{sanford2024transformersparallelcomputationlogarithmic} give an $\AHAT$ for the $k$-hop induction head task 
with gap $\gamma(n) \in \Omega(1/n^2)$ and show how to simulate it with soft attention.
As discussed in \cref{sec:introduction}, this simulation is general but not parameter-uniform.

\paragraph{$\SRASP$ simulation.}

\Citet{strobl2024transformers} consider transductions (functions from finite strings to finite strings) and introduce the language $\SRASP$ as a concise way to define them.
They show that any transduction expressible in $\SRASP$ can be computed by an $\AHAT$ 
with gap $\gamma(n) \in \Omega(1/n^3)$
and $x_{\textnormal{max}}(n) \in O(1)$.
Thus, by \cref{thm:smat_approx_hard_attention}, every transduction that is expressible in $\SRASP$ can be exactly computed by a $\tauSMAT$ with $\tau(n) \in \Omega(1/(n^3 \log n))$.

\paragraph{\prob{Dyck}-$k$ membership.} The language \prob{Dyck}-$k$ consists of balanced sequences of brackets of $k$ types.
\Citet{yao2023selfattentionnetworksprocessbounded} give a transformer for \prob{Dyck}-$k$ using selective layer norm and both soft attention and average-hard attention.
Their algorithm can be expressed in $\SRASP$ (\cref{sec:S-RASP-for-Dyck-k}), so \cref{thm:smat_approx_hard_attention} gives a $\tauSMAT$ with $\tau(n) \in \Omega(1/(n^3 \log n))$.
However, we conjecture that \prob{Dyck}-2 cannot be defined in $\LTLCMon$.

\section{Conclusion}
\label{sec:conclusion}

\emph{Under what conditions can soft attention simulate hard attention?}
The answer to this question turns out to be rather intricate.
As argued in the introduction, any such simulation must give up one of the assumptions of \citet{hahn-2020-theoretical} by adding temperature scaling, rounding, non-Lipschitz layer normalization, or unbounded position embeddings; or give up a fixed-size gap between acceptance and rejection. We have demonstrated how to use inverse-temperature scaling and unbounded position embeddings to simulate various subclasses of average-hard attention transformers, and inverse-temperature scaling to simulate general average-hard attention transformers.
Future research could explore the other interventions, as well as how all these interventions interact in practice with training dynamics.

\section*{Acknowledgments}
We thank Clayton Sanford for his generous assistance in understanding certain aspects of his paper \citep{sanford2024transformersparallelcomputationlogarithmic};
Bingbin Liu and her co-authors for their beautiful paper \citep{liu2023transformers} which helped spark the current work; and
Anej Svete for discussion about  simulation of $n$-gram models \citep{svete-cotterell-2024-transformers}.
Finally, we thank the anonymous reviewers for their helpful comments, one of which provoked a substantial improvement of \cref{thm:smat_approx_hard_attention}.

\bibliography{main}

\clearpage\appendix

\allowdisplaybreaks

\section{MLP Rounding}

\begin{lemma}
    \label{lemma-relu-mlp-rounding}
    Let $p$ and $p'$ be positive integers with $p' > p$.  Any $\ReLU$ FFN to round $p'$-bit numbers to $p$-bit numbers with $\ell$ layers and width $m$ satisfies $(2m)^{\ell} \ge 2^p$. 
\end{lemma}

\begin{proof}
      \Citet[Lemma 2.1]{telgarsky2015representationbenefitsdeepfeedforward} %
      proves that if a $\ReLU$ FFN of depth $\ell$ and width $m$ computes a function $f \colon \R \to \R$ then $f$ is continuous piecewise affine with at most $(2m)^\ell$ pieces.
      So we need to show that any piecewise affine function rounding a $p'$-bit number to a $p$-bit number requires at least $2^p$ pieces.

      Let $r$ be the function that rounds a $p'$-bit number to a $p$-bit number.  The domain of $r$ is the set of $2^{p'}$ real numbers exactly representable with $p'$ bits, and the range of $r$ is the set of $2^{p}$ real numbers exactly representable with $p$ bits.  (We assume that each $p$-bit number represents a different real value.) For each $y$ in the range of $r$, let $m(y)$ be the least $x$ in the domain of $r$ such that $r(x) = y$.

      Suppose $f \colon \R \to \R$ is a piecewise affine function that agrees with $r$ on its domain.  If $y_1 < y_2$ are two consecutive elements of the range of $r$, then $x_1 = m(y_1)$ and $x_2 = m(y_2)$ must be in different pieces of $f$.  This is because there is at least one point $x$ in the domain of $r$ with $x_1 < x < x_2$, and if $x_1$ and $x_2$ were in the same piece, then the value of $f(x)$ would fall between $y_1$ and $y_2$, which are consecutive elements of the range of $r$, a contradiction.  Thus $f$ must have at least $2^p$ pieces, and the claim follows. 
\end{proof}

Because \citet{sanford2024transformersparallelcomputationlogarithmic} consider precision $p = \Theta(\log n)$ and $p' > p$, this implies that a constant-depth $\ReLU$ MLP to round $p'$-bit numbers to $p$-bit numbers must have width at least $n^{\epsilon}$ for some $\epsilon > 0$.

\section{Approximation Bounds}

\subsection{Proof of \cref{lem:softmax_bound_constant}}
\label{sec:softmax_bound_constant_proof}

Let 
\begin{align*}
a &= e^{s_{j^*}} &
b &= \sum_{j \ne j^*} e^{s_{j}} \le \frac{2e^{-\gamma}}{1-e^{-\gamma}} a.
\end{align*}
Then
\begin{align*}
\firstline{\norm{\hardmax(\vec{s}) - \softmax\vec{s}}} \\
&= \left( 1 - \frac{a}{a+b} \right) + \frac{b}{a+b} = 2\left(1-\frac{a}{a+b}\right) \\
&\le 2\left(1-\frac{a}{a+\frac{2e^{-\gamma}}{1-e^{-\gamma}} a}\right) = 2\left(1-\frac{1}{1+\frac{2e^{-\gamma}}{1-e^{-\gamma}}}\right) \\ 
&= \frac{4e^{-\gamma}}{1+e^{-\gamma}} \le 4e^{-\gamma}. 
\end{align*}
This allows us to bound the difference between hard and soft attention:
\begin{align*}
\firstline{\left|\,\sum_j [\hardmax \vec{s}]_j \, v_j - \sum_j [\softmax \vec{s}]_j \, v_j \,\right|} \notag \\
&\le \left\| \hardmax \vec{s} - \softmax \vec{s} \right\|_1 v_{\textnormal{max}} \\
&\le 4e^{-\tlscale} v_{\textnormal{max}}.
\end{align*}

\subsection{Table Lookup (Proof of \cref{table-lookup-approx})}
\label{sec:table-lookup-approx-proof}
    Observe that $\sum_j [\hardmax \vec{s}]_j j = \tablerow$,
    and let $\vec{\alpha} = \softmax \vec{s}$, so we are trying to bound the difference $|\tablerow - \sum_j \alpha_j j|$.
    Subtracting the constant term $\tlscale \tablerow^2$ from each $s_j$ does not affect the $\alpha_j$, so we have:
    \begin{align*}
    \alpha_j &= \frac{e^{-\tlscale (j - \tablerow)^2}}{Z} &
    Z &= \sum_{k=1}^{n} e^{-\tlscale (k - \tablerow)^2} \ge 1.
    \end{align*}
    Then
    \begin{align*}
        \firstline{\left| \, \tablerow - \sum_{j=1}^{n} \alpha_j j \, \right|} \\
        &= \left| \, \sum_{j=1}^{n} \alpha_j (\tablerow - j) \, \right| \\
        &= \left| \, \sum_{j=1}^{\tablerow-1} \alpha_j (j - \tablerow) + 
        \sum_{j=\tablerow+1}^{n} \alpha_j (j - \tablerow) \, \right| \\
        &\le \sum_{j=\tablerow+1}^{n} \alpha_j (j - \tablerow) = \frac{1}{Z} \sum_{j=\tablerow+1}^{n} e^{-\tlscale (j - \tablerow)^2} (j - \tablerow) \\
        &\le \sum_{j=\tablerow+1}^{n} e^{-\tlscale (j - \tablerow)^2} (j - \tablerow) = \sum_{k=1}^{n-\tablerow} e^{-\tlscale k^2} k \\
        &\le \sum_{k=1}^{\infty} e^{-\tlscale k^2} k = e^{-\tlscale} + {\underbrace{\sum_{k=2}^{\infty} e^{-\tlscale k^2} k}_{(*)}}.
    \end{align*}
    Consider the function $f(x) = e^{-\tlscale x^2} x$, which is non-negative for $x \ge 0$ and decreasing for $x \ge 1$. So the term $(*)$ is a lower Riemann sum of~$f$:
    \begin{align*}
    \sum_{k=2}^{\infty} e^{-\tlscale k^2} k &\le \int_{1}^{\infty} \hspace*{-1em} e^{-\tlscale x^2} x \, dx = \frac{1}{2\tlscale}e^{-\tlscale} \le \frac12 e^{-\tlscale}.
    \end{align*}

\subsection{Tie-Breaking (Proof of \cref{thm:tie})}
\label{sec:tie_proof}

(\ref{thm:tie_causal}) Let $j^*$ be the rightmost position $j^* \le n$ maximizing~$s_{j^*}$. Consider any other position $j$.

If $s_j = s_{j^*}$, then $j \le j^*$, and
\begin{align*}
\hat{s}_{j^*} - \hat{s}_j &= \gamma n^2 \left(-\frac1{j^*} + \frac1{j}\right) \\ &= \gamma n^2 \frac{j^*-j}{j^*j} \\ &\ge \gamma |j^*-j|.
\end{align*}

If $s_j \le s_{j^*} - 1$, then
\begin{align*}
\hat{s}_{j^*} - \hat{s}_j &\ge \gamma n^2 \left(1 - \frac1{j^*} + \frac1{j}\right) \\ &\ge \gamma n^2 \left(1 - \frac11 + \frac1{n}\right) \\ &= \gamma n \ge \gamma |j^*-j|.
\end{align*}

So we can use \cref{lem:softmax_bound_constant} with a gap of $\gamma$:
\[\left|\,\sum_j [\rhardmax \vec{s} - \softmax \hat{\vec{s}}]_j v_j\,\right| \le 4e^{-\gamma} v_{\textnormal{max}}.\]

(\ref{thm:tie_rightmost}) Let $j^*$ be the rightmost position $j^* \le n$ maximizing~$s_{j^*}$. Consider any other position $j$.

If $s_j = s_{j^*}$, then $j \le j^*$, and
\begin{align*}
\hat{s}_{j^*} - \hat{s}_j &= 2\gamma n\left(\frac{j^*}{2n} - \frac{j}{2n}\right) \\ &= \gamma |j^*-j|.
\end{align*}

If $s_j \le s_{j^*} - 1$, then 
\begin{align*}
\hat{s}_{j^*} - \hat{s}_j &\ge 2\gamma n \left(1 + \frac{j^*}{2n} - \frac{j}{2n}\right) \\ &\ge 2\gamma n \left(1 + \frac{1}{2n} - \frac{n}{2n} \right) \\ &= \gamma (1+n) \ge \gamma |j^*-j|.
\end{align*}

So we can use \cref{lem:softmax_bound_constant}:
\[ \left|\,\sum_j [\rhardmax \vec{s} - \softmax \hat{\vec{s}}]_j v_j\,\right| \le 4e^{-\gamma} v_{\textnormal{max}}.\]

(\ref{thm:tie_leftmost}) The leftmost case is symmetric.

\section{Proofs for Average-Hard Attention}
\label{sec:proofs_for_ut_attention}

\renewcommand{\invec}{\vec{x}}

To reduce complication, we assume that the transformers we consider have 
$d = d_{\textrm{k}} = d_{\textrm{f}}$.
We begin with a lemma bounding the effect of an affine transformation on the norm of a vector.
\begin{lemma}
\label{lem:l1_norm_bound}
    Let $\mat{W} \in \R^{d \times d}$ and $\vec{b}, \invec \in \R^d$.  Then $\norm{\mat{W}\invec + \vec{b}} \le dw_{\textnormal{max}} \norm{\invec} + db_{\textnormal{max}}$, where $w_{\textnormal{max}}$ and $b_{\textnormal{max}}$ are the maximum absolute values of any entry in $\mat{W}$ and $\vec{b}$, respectively.
\end{lemma}
\begin{proof}
\begin{align*}
\norm{\mat{W}\invec + \vec{b}} & \le \sum_{k=1}^d |\vec{W}_{k,*} \cdot \invec| + \norm{\vec{b}}\\
                                & \le \sum_{k=1}^d w_{\textnormal{max}} \norm{\invec} + db_{\textnormal{max}}\\
                                &= dw_{\textnormal{max}} \norm{\invec} + db_{\textnormal{max}}. \tag*{\qedhere}
\end{align*}
\end{proof}

For the rest of this section, as defined in \cref{sec:uuhat_to_smat_temp}, $T$ is an $\AHAT$ with $L$ layers and gap function $\gamma(n)$, and $\hat{T}$ is the $\tauSMAT$ with the same parameters.  
Let $p_{\textnormal{max}}$ be the maximum of $1$ and the absolute value of any parameter occurring in $T$.
Over all inputs of length $n$,
let $x_{\textnormal{max}}(n)$ be the maximum of $1$ and the absolute value of any entry in any of the $\vec{x}_i$. 

\renewcommand{\invec}{\vec{g}}
\renewcommand{\outvec}{\vec{h}}

First, we bound the activations of $T$.
\begin{lemma}
\label{lemma:activation-bound}
    Let
    $K = 2(d^2p_{\textnormal{max}}^2 + 1)(dp_{\textnormal{max}} + 1)$.  For any $\ell \in \{0,\ldots,L\}$,
    let $(\outvec_1,\ldots,\outvec_n)$ be the sequence of activations 
    after layer $\ell$ for some input of length $n$. Then
    for all $i \in [n]$, $\norm{\outvec_i} \le K^{\ell} d x_{\textnormal{max}}(n) 
    \le K^{L} d x_{\textnormal{max}}(n)$.
\end{lemma}

\begin{proof}
    By induction on $\ell$.
    Fix an input of length $n$.
    For $\ell = 0$, we have $\outvec_i = \vec{x}_i$ for the initial activation vector $\vec{x}_i$, and $\norm{\vec{x}_i} \le d x_{\textnormal{max}}(n)$.
    
    For $\ell > 0$, let $(\invec_1,\ldots,\invec_n)$ be the sequence of inputs to layer $\ell$.  By the induction hypothesis, for all $i \in [n]$, $\norm{\invec_i} \le K^{\ell-1} d x_{\textnormal{max}}(n)$.
    For $i \in [n]$, the attention layer computes 
    \[\vec{c}_i = \invec_i + \sum_{j=1}^{n} \alpha_{ij} \vproj \invec_j.\] 
    By \cref{lem:l1_norm_bound}, $\norm{\vproj \invec_j} \le dp_{\textnormal{max}} \norm{\invec_j}$  and thus $\norm{\vec{c}_i} \le (dp_{\textnormal{max}} + 1)K^{\ell-1} d x_{\textnormal{max}}(n)$.

    Let $\ffwone$, $\vec{b}_1$, $\ffwtwo$, $\vec{b}_2$ be the parameters of the feedforward layer. We have \[ \outvec_i = \vec{c}_i + \ffwtwo\, \ReLU(\ffwone \vec{c}_i + \vec{b}_1) + \vec{b}_2. \] Because the $\ReLU$ operation does not increase the norm of a vector, it suffices to bound, using \cref{lem:l1_norm_bound},
    \begin{align*}
    \norm{\outvec_i} &\le \norm{\vec{c}_i + \ffwtwo(\ffwone \vec{c}_i + \vec{b}_1) + \vec{b}_2} \\
    &\le (d^2p_{\textnormal{max}}^2 + 1) \norm{\vec{c}_i} + d^2p_{\textnormal{max}}^2 + dp_{\textnormal{max}}.
    \end{align*}
    Substituting in the bound for $\norm{\vec{c}_i}$ and observing that 
    $(d^2p_{\textnormal{max}}^2 + 1)(dp_{\textnormal{max}} + 1) \ge d^2p_{\textnormal{max}}^2 + dp_{\textnormal{max}}$ and 
    $x_{\textnormal{max}}(n) \ge 1$, we conclude that
    \begin{align*}
    \norm{\outvec_i} &\le 2(d^2p_{\textnormal{max}}^2 + 1)(dp_{\textnormal{max}} + 1) K^{\ell-1}dx_{\textnormal{max}}(n)\\
                    &= K^{\ell}dx_{\textnormal{max}}(n).\tag*{\qedhere}
    \end{align*}
\end{proof}

\subsection{FFNN Layers}
\label{sec:ffnn-layers}

\renewcommand{\invec}{\vec{c}}
\renewcommand{\outvec}{\vec{h}}

A FFNN layer computes the same function in $T$ and $\hat{T}$. It may amplify error in the input values, but does not introduce new error.
\begin{lemma}
\label{lem:ffnn_error_bound}
    If for a FFNN layer of $T$ we have two sequences of input vectors $(\invec_1,\ldots,\invec_n)$ and $(\hat{\invec}_1,\ldots,\hat{\invec}_n)$ such that for all $i$, $\norm{\invec_i - \hat{\invec}_i} \le \epsilon$, then for the sequences of output vectors $(\outvec_1,\ldots,\outvec_n)$ of this layer of $T$ and $(\hat{\outvec}_1,\ldots,\hat{\outvec}_n)$ of the same layer of $\hat{T}$, we have for all $i$, $\norm{\outvec_i - \hat{\outvec}_i} \le (d^2p_{\textnormal{max}}^2 + 1) \epsilon$. 
\end{lemma}

\begin{proof}
Let $\mat{W}_1$, $\vec{b}_1$, $\mat{W}_2$, $\vec{b}_2$ be the parameters of the FFNN.
We have $|\ReLU(x) - \ReLU(y)| \le |x - y|$ for all real numbers $x, y$, so the $\ReLU$ operation does not increase error. Thus it suffices to bound the increase in error from the two affine transformations and the residual connection.

For  $i \in [n]$, $\norm{(\ffwone \invec_i + \vec{b}_1) - (\ffwone\hat{\invec}_i + \vec{b}_1)} = \norm{\ffwone(\invec_i - \hat{\invec}_i)} \le dp_{\textnormal{max}}\epsilon$, by \cref{lem:l1_norm_bound}.

The second affine transformation with parameters $\ffwtwo$ and $\vec{b}_2$ at worst multiplies the error by $dp_{\textnormal{max}}$ again, and the residual connection adds at most the input error, so the final error is bounded by $(d^2p_{\textnormal{max}}^2 +1) \epsilon$.
\end{proof}

\subsection{Self-attention layers}

A self-attention layer may not only amplify error in the input values, but may also introduce new error because of the difference between average-hard attention and temperature scaled softmax attention.
Consider a self-attention layer of $T$ with parameter matrices $\qproj$, $\kproj$, and  $\vproj $ in $\R^{d \times d}$.
We first bound the error in computing attention scores~$s_{i,j}$.
\renewcommand{\invec}{\vec{h}}
\renewcommand{\outvec}{\vec{c}}
\begin{lemma}
\label{lemma:attention-score-error}
Let $(\invec_1,\ldots,\invec_n)$ be a sequence of input vectors to this attention layer for $T$ and $(\hat{\invec}_1,\ldots,\hat{\invec}_n)$ a sequence of input vectors to the same attention layer for $\hat{T}$, where for all $i$,
$\norm{\invec_i - \hat{\invec}_i} \le \epsilon$ for some $\epsilon \le 1$.
There exists a constant $K \ge 1$ depending on $d$ and $p_{\textnormal{max}}$ such that for the attention scores $s_{i,j}$ we have for all $i, j$,
\[|s_{i,j} - \hat{s}_{i,j}| \le 
 K h_{\textnormal{max}} \epsilon\]
where $h_{\textnormal{max}}$ is the maximum of $1$ and the absolute value of any entry in any $\invec_j$.
\end{lemma}

\begin{proof} 
The error in the query and key vectors at positions $i$ and $j$, respectively, are, by \cref{lem:l1_norm_bound},
\begin{align*}
\norm{\vec{q}_i - \hat{\vec{q}}_i} &=
\norm{\qproj \vec{h}_i - \qproj \hat{\vec{h}}_i} \le dp_{\textnormal{max}} \epsilon \\
\norm{\vec{k}_j - \hat{\vec{k}}_j} &=
\norm{\kproj \vec{h}_j - \kproj \hat{\vec{h}}_j} \le dp_{\textnormal{max}} \epsilon.
\end{align*}
Then the error in the attention scores is
\begin{align*}
\firstline{\left|s_{i,j} - \hat{s}_{i,j}\right|} \\ &= \frac1{\sqrt d} \left|\vec{q}_i\cdot\vec{k}_j - \hat{\vec{q}}_i \cdot \hat{\vec{k}}_j\right| \\
&= \frac1{\sqrt d} \bigl|\vec{q}_i\cdot\vec{k}_j - \begin{aligned}[t] &(\vec{q}_i + (\hat{\vec{q}}_i-\vec{q}_i)) \\ &{} \cdot (\vec{k}_j + (\hat{\vec{k}}_j-\vec{k}_j))\bigr|\end{aligned} \\
&= \frac1{\sqrt d} \bigl| \begin{aligned}[t] &- \vec{q}_i \cdot (\hat{\vec{k}}_j-\vec{k}_j) - (\hat{\vec{q}}_i - \vec{q}_i) \cdot \vec{k}_j \\ &\quad - (\hat{\vec{q}}_i-\vec{q}_i) \cdot (\hat{\vec{k}}_j-\vec{k}_j) \bigr|\end{aligned},
\intertext{and observing that $|\vec{a}\cdot\vec{b}| \le \norm{\vec{a}} \norm{\vec{b}}$ and $\epsilon \le 1 \le h_{\textnormal{max}}$, we get}
\firstline{\left|s_{i,j} - \hat{s}_{i,j}\right|} \\
&\le \frac1{\sqrt{d}}(2d^2p_{\textnormal{max}}^2 h_{\textnormal{max}} \epsilon + d^2p_{\textnormal{max}}^2 \epsilon^2) \\
&\le 3 d^{3/2} p_{\textnormal{max}}^2 h_{\textnormal{max}} \epsilon.
\tag*{\qedhere}
\end{align*}
\end{proof}
We now turn to bounding the error in the resulting attention weights $\alpha_{i,j}$.
\begin{lemma}
\label{lem:softmax_approximation_bound_tie}
Let $\tau > 0$.
Let $\vec{s} = (s_1, \ldots, s_n)$ be a sequence of scores for $T$ with gap $\gamma$, and let $\hat{\vec{s}} = (\hat{s}_1, \ldots \hat{s}_n)$ be a sequence of scores for $\hat{T}$, respectively, such that for all $i$, $|s_i - \hat{s}_i| \le \epsilon$. Then
\[ \norm{\avghardmax \vec{s} - \softmax_\tau \hat{\vec{s}}} \le 2ne^{-\gamma/\tau} + 4\epsilon/\tau. \]
\end{lemma}

\begin{proof}
There are two sources of error, the softmax and the approximation of $\vec{s}$:
\begin{align*}
\firstline{\norm{\avghardmax \vec{s} - \softmax_\tau \hat{\vec{s}}}} \\
&\le \norm{\avghardmax \vec{s} - \softmax_\tau \vec{s}} \\ &\quad + \norm{\softmax_\tau \vec{s} - \softmax_\tau \hat{\vec{s}}}.
\end{align*}
First, by Lemma B.7 of \citet{edelman2022inductive},
\begin{align*}
\norm{\avghardmax \vec{s} - \softmax \vec{s}} &\le 2ne^{-\gamma/\tau}.
\end{align*}
Second (cf.~Theorem 14f of \citet{chiang:2025}), for all $i \in [n]$,
\begin{align*}
[\softmax_\tau \hat{\vec{s}}]_i &=
\frac{e^{\hat{s}_i/\tau}}{\sum_j e^{\hat{s}_j/\tau}} \\ &\le \frac{e^{(s_i + \epsilon)/\tau}}{\sum_j e^{(s_j - \epsilon)/\tau}} 
= \frac{e^{(s_i + 2\epsilon)/\tau}}{\sum_j e^{s_j/\tau}} \\
&= e^{2\epsilon/\tau} [\softmax_\tau \vec{s}]_i.
\end{align*}
If $\epsilon \le \tau/2$, then
\begin{align*}
[\softmax_\tau \hat{\vec{s}}]_i&\le (1+4\epsilon/\tau) [\softmax_\tau \vec{s}]_i
\intertext{and by similar reasoning,}
[\softmax_\tau \hat{\vec{s}}]_i &\ge (1-4\epsilon/\tau) [\softmax_\tau \vec{s}]_i.
\end{align*}
Then
\begin{align*}
\firstline{\norm{\softmax_\tau \vec{s} - \softmax_\tau \hat{\vec{s}}}} \\
&\le \sum_i 4\epsilon/\tau [\softmax_\tau \vec{s}]_i = 4\epsilon/\tau.
\end{align*}
On the other hand, if $\epsilon > \tau/2$, then $4\epsilon/\tau > 2$, and $\norm{\softmax_\tau \vec{s} - \softmax_\tau \hat{\vec{s}}}$ cannot exceed $2$.
\end{proof}

The next stage of an attention layer averages the vectors $\vproj \invec_j$ 
weighted by the attention weights $\alpha_{i,j}$ and adds the residuals.
The following bounds the error in the result in terms of the errors in the approximations of $\alpha_{i,j}$ and $\invec_j$.
\begin{lemma}
\label{lem:weighted_sum_error_bound}
    Let $(\invec_1,\ldots,\invec_n)$ and $(\hat{\invec}_1,\ldots,\hat{\invec}_n)$ be sequences of elements of $\R^d$ such that $\norm{\invec_i - \hat{\invec}_i} \le \epsilon$ for all $i$.  For each $i$, let $(\alpha_{i,1},\ldots,\alpha_{i,n})$ and $(\hat{\alpha}_{i,1},\ldots,\hat{\alpha}_{i,n})$ be sequences of real numbers in $[0,1]$ that each sum to $1$ such that $\norm{\alpha_{i,*} - \hat{\alpha}_{i,*}} \le \epsilon_1$.
    Let $\outvec_i = \sum_{j=1}^{n} \alpha_{i,j} \vproj \invec_j +\invec_i$ and $\hat{\outvec}_i =\sum_{j=1}^{n} \hat{\alpha}_{i,j} \vproj  \hat{\invec}_j  +\hat{\invec}_i$.
    Then for all $i$,
\[\norm{\outvec_i - \hat{\outvec}_i} \le dp_\textnormal{max}(d h_\textnormal{max} \epsilon_1 + \epsilon) +\epsilon\]
where $h_\textnormal{max}$ is the maximum absolute value of any entry in any $\invec_j$. 
\end{lemma}
\begin{proof}
    Fix position $i$.  Then
    \[\outvec_i - \hat{\outvec}_i = \vproj  \sum_{j=1}^n (\alpha_{i,j}\invec_j - \hat{\alpha}_{i,j}\hat{\invec}_j) + \invec_i - \hat{\invec}_i.\]
    Fix position $j$.  Then
    \begin{align*}
    \firstline{\norm{\alpha_{i,j}\invec_j - \hat{\alpha}_{i,j}\hat{\invec}_j}} \\
        & = \norm{\alpha_{i,j}\invec_j - \hat{\alpha}_{i,j}\invec_j + \hat{\alpha}_{i,j}\invec_j - \hat{\alpha}_{i,j}\hat{\invec}_j} \\
        & \le |\alpha_{i,j} - \hat{\alpha}_{i,j}| \norm{\invec_j} + \hat{\alpha}_{i,j} \norm{\invec_j - \hat{\invec}_j} \\
        & \le |\alpha_{i,j} - \hat{\alpha}_{i,j}| \norm{\invec_j} + \hat{\alpha}_{i,j} \epsilon.
    \end{align*}
    This can be used to show
    \begin{align*}
    \firstline{\norm{\,\sum_{j=1}^n (\alpha_{i,j}\invec_j - \hat{\alpha}_{i,j}\hat{\invec}_j)\,}} \\
      &\le \sum_{j=1}^n \norm{\alpha_{i,j}\invec_j - \hat{\alpha}_{i,j}\hat{\invec}_j} \\
      &\le \sum_{j=1}^n \left(|\alpha_{i,j} - \hat{\alpha}_{i,j}| \norm{\invec_j} + \hat{\alpha}_{i,j} \epsilon\right) \\
      &\le d h_\textnormal{max} \epsilon_1 + \epsilon.
    \end{align*}  
    Finally, the multiplication by $\vproj $ multiplies this bound by at most $dp_\textnormal{max}$, and the residual connection adds at most another $\epsilon$.
\end{proof}

\subsection{Proof of \cref{lem:one-layer-error-tie}}
\label{sect:proof-lemma-one-layer-error}

\renewcommand{\invec}{\vec{g}}
\renewcommand{\outvec}{\vec{h}}

\begin{proof}
Let $\invec_i$, $\hat{\invec}_i$, $\outvec_i$, and $\hat{\outvec}_i$ for $i \in [n]$ be as in the statement of the lemma.
Assume that for all~$i \in [n]$, $\norm{\invec_i - \hat{\invec}_i} \le \epsilon \le 1$.

By \cref{lemma:activation-bound,lemma:attention-score-error},
there exist constants $K_1, K_2 \ge 1$ (not depending on $\ell$) such that for $i,j \in [n]$,
\begin{align*}
|s_{i,j} - \hat{s}_{i,j}| &\le K_1 h_{\textnormal{max}}(n) \epsilon
\\
&\le K_2 x_{\textnormal{max}}(n) \epsilon.
\end{align*}
By \cref{lem:softmax_approximation_bound_tie}, 
for all $i \in [n]$,
\[ \norm{\alpha_{i,*} - \hat{\alpha}_{i,*}} \le 2ne^{-\frac{\gamma(n)}{\tau(n)}} + \tfrac{4K_2x_{\textnormal{max}}(n)\epsilon}{\tau(n)}. \]
By \cref{lemma:activation-bound,lem:weighted_sum_error_bound}, 
there exist constants $K_3, K_4, K_5 \ge 1$ (not depending on $\ell$) such that for all $i \in [n]$,
\begin{align*}
\firstline{\norm{\vec{c}_i - \hat{\vec{c}}_i}} \\ &\le K_3 x_{\textnormal{max}} \left( ne^{-\frac{\gamma(n)}{\tau(n)}} + \tfrac{x_{\textnormal{max}}(n)\epsilon}{\tau(n)}\right) + K_4 h_{\textnormal{max}} \\
&\le K_5 x_{\textnormal{max}} \left( ne^{-\frac{\gamma(n)}{\tau(n)}} + \tfrac{x_{\textnormal{max}}(n)\epsilon}{\tau(n)}\right).
\end{align*}
Finally, by \cref{lem:ffnn_error_bound}, there exists a constant $K \ge 1$ (not depending on $\ell$) such that
\begin{align*}
\norm{\outvec_i - \hat{\outvec}_i} \le K x_{\textnormal{max}}(n) \left( n e^{-\frac{\gamma(n)}{\tau(n)}} + \tfrac{x_{\textnormal{max}}(n) }{\tau(n)} \epsilon \right). \tag*{\qedhere}
\end{align*}    
\end{proof}

\subsection{Proof of \cref{thm:smat_approx_hard_attention}}
\label{sec:smat_approx_hard_attention_proof}

\begin{proof}

Fix an input length $n$.
For brevity, we write $\min(\gamma(n),1)$, $\max(x_{\textnormal{max}}(n),1)$, and $\min(\tau(n),1)$ as $\gamma,x_{\textnormal{max}}$, and $\tau$, respectively.

Fix an input sequence of vectors $(\vec{x}_1,\ldots,\vec{x}_n)$. 
Let us write the bound from \cref{lem:one-layer-error-tie} as
\begin{align*}
\epsilon_\ell &= \norm{\vec{h}_i - \hat{\vec{h}}_i} \le a + r \epsilon_{\ell-1} \\
a &= Kx_{\textnormal{max}} n e^{\gamma/\tau} \\
r &= Kx_{\textnormal{max}}^2/\tau.
\end{align*}
Since the error $\epsilon_\ell$ at each layer is bounded by an affine function of the error at the layer below, the final error is the sum of a geometric series:
\begin{align*}
\epsilon_L = \norm{\vec{y}_i - \hat{\vec{y}}_i} &= a + ar + \cdots + ar^{L-1} \\
&\le Lar^{L-1} \\
&= L K nx_{\textnormal{max}} e^{-\gamma/\tau} \left(\frac{Kx_{\textnormal{max}}^2}{\tau}\right)^{L-1}.
\end{align*}

Taking logs of both sides and using the inequality $\log r \le \log r_0 + \frac{1}{r_0}(r-r_0)$ based on the first-order Taylor approximation, with $r$ as above
and $r_0 = 2K(L-1)x_{\textnormal{max}}^2/\gamma$, we get

\begin{align*}
\firstline{\log \epsilon_L = \log \norm{\vec{y}_i - \hat{\vec{y}}_i}} \\
& \le \log Knx_{\textnormal{max}} - \frac{\gamma}{\tau} + (L-1) \log
r
\\
& \le \log Knx_{\textnormal{max}} - \frac{\gamma}{\tau} + (L-1) \left(\log r_0 + \frac{r-r_0}{r_0}\right) \\
&\le \log Knx_{\textnormal{max}} - \frac{\gamma}{2\tau} +
(L-1) (\log r_0-1).
\end{align*}
We want $\epsilon_L$ to be bounded by the given error function $\epsilon(n)$.
If $\alpha$ and $\beta$ are such that $\epsilon(n) \le \alpha/n^\beta$, then we need 
\begin{align*}
    \firstline{\log Knx_{\textnormal{max}} - \frac{\gamma}{2\tau} +
(L-1) (\log r_0-1)} \\ &\leq \log \alpha/n^\beta
\end{align*}
which is obtained with the temperature bound
\begin{align*}
\frac1\tau &\ge \frac2\gamma \left(\log \frac{K n^{\beta+1}x_{\textnormal{max}}}{\alpha} + (L-1)
(\log r_0 - 1)
\right) \\
&\in O\left(\frac1\gamma \log \frac{n x_{\textnormal{max}}}\gamma \right). \tag*{\qedhere}
\end{align*}
\end{proof}

\section{$\SRASP$ Program for \prob{Dyck}-$k$}
\label{sec:S-RASP-for-Dyck-k}
The following $\SRASP$ program \cite{strobl2024transformers} returns, for each prefix of the input ($\inputx$), whether the prefix belongs to \prob{Dyck}-$k$, the language of balanced strings of brackets with $k$ types of brackets.
We assume that $\vecname{left}(\sigma)$ returns $1$ if $\sigma$ is a left bracket and $0$ otherwise, $\vecname{right}(\sigma)$ returns $1$ if $\sigma$ is a right bracket and $0$ otherwise, and $\vecname{mismatch}(\sigma,\tau)$ returns $\true$ if $\sigma$ and $\tau$ are not left and right brackets of the same type.  

\begin{raspcode}
    \begin{align*}
        \vecname{sleft}(i) &= \attsum{j}{j \le i}{\vecname{left}(\inputx(i))}\\
         \vecname{sright}(i) &= \attsum{j}{j \le i}{\vecname{right}(\inputx(i))}\\
         \vecname{er1}(i) &= \vecname{sright}(i) > \vecname{sleft}(i) \\
         \vecname{diff}(i) &= \vecname{sleft}(i) - \vecname{sright}(i) \\
         \vecname{d}(i) &= \vecname{diff}(i) + \vecname{right}(\inputx(i))\\
         \vecname{check}(i) &= \attrdefault{j}{j < i}{\vecname{d}(j) = \vecname{d}(i)} {\inputx(j)}{\rechar{?}} \\
         \vecname{er2}(i) &= (\vecname{right}(\inputx(i)) = 1) \land \\
         & \qquad \vecname{mismatch}(\vecname{check}(i),\inputx(i))\\
         \vecname{okprefix}(i) &= \attrdefault{j}{j \le i}{\vecname{er1}(j) \lor \vecname{er2}(j)}{\false}{\true}\\
         \outputy(i) &= \vecname{okprefix}(i) \land (\vecname{diff}(i) = 0)
    \end{align*}
\end{raspcode}

$\vecname{sleft}(i)$ and $\vecname{sright}(i)$ are the numbers of left and right (resp.)~brackets of any type through position $i$. An error ($\vecname{er1}(i)$) is recorded wherever there are more right than left brackets through position $i$. $\vecname{d}(i)$ is the depth of the bracket at~$i$. $\vecname{check}(i)$  records the nearest bracket strictly to the left that has the same depth. An error ($\vecname{er2}(i)$) is recorded wherever there is a right bracket and $\vecname{check}(i)$ is not a left bracket of the same type. Finally, the output $\outputy(i)$ for $i$ is $\true$ iff the prefix through $i$ has no errors ($\vecname{okprefix}(i)$) and contains an equal number of left and right brackets.

\clearpage

\end{document}